\RequirePackage{silence}
\documentclass[12pt]{article}
\usepackage{amsmath,amssymb,amsthm,mathrsfs,bbm,bm}
\usepackage{graphicx,subcaption,float,enumerate}
\usepackage[pdfpagelabels]{hyperref}
\usepackage[capitalise]{cleveref}  %
\usepackage[natbibapa]{apacite}  %

\usepackage{algorithm}
\usepackage{algpseudocode}

\newtheorem{theorem}{Theorem}[section]
\newtheorem{lemma}[theorem]{Lemma}
\newtheorem{definition}[theorem]{Definition}
\newtheorem{corollary}[theorem]{Corollary}
\newtheorem{remark}{Remark}

\newtheorem{assumption}{Assumption}

\usepackage[sharp]{easylist}
\newcommand{\bel}{\begin{easylist}[itemize]}
\newcommand{\eel}{\end{easylist}}
\makeatletter
\newcommand{\el@aux}[1]{\begin{easylist}[itemize] #1 \end{easylist}\endgroup}
\newcommand{\el}{\begingroup\Activate\el@aux}
\newcommand{\eln@aux}[1]{\begin{easylist}[enumerate] #1 \end{easylist}\endgroup}
\newcommand{\eln}{\begingroup\Activate\eln@aux}
\makeatother

\newcommand{\one}{\textbf{1}}

\newcommand{\bt}[1]{\textbf{#1}}
\newcommand{\itt}[1]{\textit{#1}}

\newcommand{\paren}[1]{\left(#1\right)}

\newcommand{\eq}[1]{\begin{equation*}\begin{aligned}#1\end{aligned}\end{equation*}}
\newcommand{\eqn}[1]{\begin{equation}\begin{aligned}#1\end{aligned}\end{equation}}

\newcommand{\psd}{\succeq}
\DeclareMathOperator{\tr}{\bt{tr}}

\newcommand{\half}{\frac{1}{2}}

\newcommand{\reals}{\mathbb{R}}
\newcommand{\kldist}[2]{D\! \left( #1\|#2 \right)}

\newcommand*{\vertbar}{\rule[-1ex]{0.5pt}{2.5ex}}

\newcommand{\tB}{\tilde{B}}

\newcommand{\cB}{\mathcal{B}}

\newcommand{\cD}{\mathcal{D}}

\newcommand{\cJ}{\mathcal{J}}

\newcommand{\cU}{\mathcal{U}}

\newcommand{\cX}{\mathcal{X}}

\newcommand{\bE}{\mathbb{E}}

\newcommand{\bN}{\mathbb{N}}

\newcommand{\bR}{\mathbb{R}}

\newcommand{\aP}{\alphabet{P}}

\newcommand{\aW}{\alphabet{W}}

\newcommand{\parenth}[1] {\left(#1\right)}
\newcommand{\braces}[1] {\left\{#1\right\}}
\newcommand{\brackets}[1] {\left[#1\right]}
\newcommand{\abs}[1] {\left|#1\right|}

\newcommand{\bitm}{\begin{itemize}}
\newcommand{\eitm}{\end{itemize}}

\newcommand{\benum}{\begin{enumerate}}
\newcommand{\eenum}{\end{enumerate}}

\newcommand{\beqa}{\begin{eqnarray}}
\newcommand{\eeqa}{\end{eqnarray}}
\newcommand{\beqas}{\begin{eqnarray*}}
\newcommand{\eeqas}{\end{eqnarray*}}

\newcommand{\alphabet}[1] { {\mathsf #1}}
\newcommand{\probSimplex}[1]{ \alphabet{P}\parenth{\alphabet{#1}}}

\newcommand{\jac}{J}

\newcommand{\detSysx}{\det\parenth{\jac_{S^*_{(y)}(x)}}}

\def\argmin{\mathop{\arg\,\min}\limits}%

\newcommand{\MD}{\mathcal{M}\mathcal{D}}

\renewcommand{\bE}{\mathbb{E}}
\newcommand{\D}{\cD}                              %
\renewcommand{\MD}{\D_+}                %
\newcommand{\KRMD}{\cD_+^{KR}}          %

\newcommand{\vecj}{\mathbf{j}}              %
\newcommand{\veci}{\mathbf{i}}              %
\newcommand{\vecx}{\mathbf{x}}              %

\newcommand{\weightRow}{\bm{w}}                           %
\newcommand{\captionfonts}{\normalsize}

\makeatletter  
\long\def\@makecaption#1#2{%
  \vskip\abovecaptionskip
  \sbox\@tempboxa{{\captionfonts #1: #2}}%
  \ifdim \wd\@tempboxa >\hsize
    {\captionfonts #1: #2\par}
  \else
    \hbox to\hsize{\hfil\box\@tempboxa\hfil}%
  \fi
  \vskip\belowcaptionskip}
\makeatother   
\begin{document}
\hspace{13.9cm}1

\ \vspace{15mm}\\

{\LARGE A Distributed Framework for the Construction of Transport Maps}

\ \\
{\bf \large Diego A. Mesa$^{\displaystyle 1 *}$}\\
{\bf \large Justin Tantiongloc$^{\displaystyle 2 *}$}\\
{\bf \large Marcela Mendoza$^{\displaystyle 3 *}$}\\
{\bf \large Todd P. Coleman$^{\displaystyle 3}$}\\\\
{* These authors contributed equally to this work.}\\
{$^{\displaystyle 1}$ Departments of Electrical Engineering and Computer Science (EECS) and Biomedical Informatics (DBMI), Vanderbilt University, Nashville, TN 37205, U.S.A.}\\
{$^{\displaystyle 2}$ Department of Computer Science and Engineering, University of California, San Diego, La Jolla, CA 92093, U.S.A.}\\
{$^{\displaystyle 3}$ Department of Bioengineering, University of California, San Diego, La Jolla, CA 92093, U.S.A.}\\


\thispagestyle{empty}
\markboth{}{NC instructions}
\ \vspace{-0mm}\\
\begin{center} {\bf Abstract} \end{center}
The need to reason about uncertainty in large, complex, and multi-modal datasets has become increasingly common across modern scientific environments.  The ability to transform samples from one distribution $P$ to another distribution $Q$ enables the solution to many problems in machine learning (e.g. Bayesian inference, generative modeling) and has been actively pursued from theoretical, computational, and application perspectives across the fields of information theory, computer science, and biology.  Performing such transformations, in general, still leads to computational difficulties, especially in high dimensions.   Here, we consider the problem of computing such ``measure transport maps'' with efficient and parallelizable methods.  Under the mild assumptions that $P$ need not be known but can be sampled from, and that the density of $Q$ is known up to a proportionality constant, and that $Q$ is log-concave, we provide in this work a convex optimization problem pertaining to relative entropy minimization.  We show how an empirical minimization formulation and polynomial chaos map parameterization can allow for learning a transport map between $P$ and $Q$ with distributed and scalable methods.   We also leverage findings from nonequilibrium thermodynamics to represent the transport map as a composition of simpler maps, each of which is learned sequentially with a transport cost regularized version of the aforementioned problem formulation.   We provide examples of our framework within the context of Bayesian inference for the Boston housing dataset and generative modeling for handwritten digit images from the MNIST dataset.
\section{Introduction}
\label{seca:introduction}
While scientific problems of interest continue to grow in size and complexity, managing uncertainty is increasingly paramount.  As a result, the development and use of theoretical and numerical methods to reason in the face of
uncertainty, in a manner that can accommodate large datasets,  has been the focus of sustained
research efforts in statistics, machine learning, information theory and computer
science.   The ability to construct a mapping which transforms samples from one distribution $P$ to another distribution $Q$ enables the solution to many problems in machine learning.  

One such problem is Bayesian inference, \citep{gelman2014bayesian,bernardo2001bayesian,sivia2006data}, where a latent signal of interest is observed through noisy observations.  Fully characterizing the posterior distribution is in general notoriously challenging, due to the need to calculate the normalization constant pertaining to the posterior density.  Traditionally, point estimation procedures are used, which obviate the need for this calculation, despite their inability to quantify uncertainty.  Generating samples from the posterior distribution enables approximation of any conditional expectation, but this is typically performed with Markov Chain Monte Carlo (MCMC) methods \citep{gilks2005markov,andrieu2003introduction, hastings1970monte,
geman1984stochastic, Liu2008} despite the following drawbacks: (a) the convergence rates and mixing
times of the Markov chain are generally unknown, thus leading to practical
shortcomings like ``sample burn in'' periods; and (b) the samples generated are
necessarily correlated, lowering effective sample sizes and propagating errors
throughout estimates \citep{robert2004monte}.  If we let  $P$  be the prior distribution and $Q$ the posterior distribution for Bayesian inference , then an algorithm which can transform independent samples from $P$ to $Q$, without knowledge of the normalization constant in the density of $Q$, enables calculation of any conditional expectation with fast convergence. 

As another example, generative modeling problems entail observing a large dataset with samples from an unknown distribution $P$ (in high dimensions) and attempting to learn a representation or model so that new independent samples from $P$ can be generated.   Emerging approaches to generative modeling rely on the use of deep neural networks and include variational autoencoders \citep{kingma2013auto}, generative adversarial networks \citep{goodfellow2014generative} and their derivatives \citep{li2015generative}, and auto-regressive neural networks \citep{larochelle2011neural}. These models have led to impressive results in a number of applications, but their tractability and theory are still not fully developed.  If $P$ can be transformed into a known and well-structured distribution $Q$ (e.g. a multivariate standard Gaussian), then the inverse of the transformation  can be used to transform new independent samples from $Q$ into new samples from $P$.

While these issues relate to the functional attractiveness of the ability to characterize and sample from non-trivial distributions, there is also the issue of computational efficiency.  There continues to be an ongoing upward trend of the availability of distributed and hardware-accelerated computational resources.  As such, it would be especially valuable to develop solutions to these problems that are not only satisfactory in a functional sense, but are also capable of taking advantage of the ever-increasing scalability of parallelized computational capability.

\subsection{Main Contribution} 
\label{sub:main_contribution}

The main contribution of this work is to extend our previous results on finding transport maps to provide
a more general transport-based push-forward theorem for pushing independent samples from a distribution $P$ to independent samples from a
distribution $Q$. Moreover, we show how given only independent samples from $P$, knowledge of $Q$ up to a normalization constant, and under the traditionally mild assumption
of the log-concavity of $Q$, it can be carried out in a \itt{distributed} and
\itt{scalable} manner, leveraging the technique of alternating direction method
of multipliers (ADMM) \citep{boyd2011distributed}.   We also 
leverage variational principles from nonequilibrium thermodynamics \citep{jordan1998variational} to represent a transport map as an aggregate composition of simpler maps, each of which minimizes a relative entropy along with a transport-cost-based regularization term.  Each map can be constructed with a complementary, ADMM-based formulation, resulting in the construction of a measure transport map smoothly and sequentially with applicability in high-dimensional settings.  

 Expanding on previous work on the real-world applicability of these general-purpose algorithms,  we showcase the implementation of a Bayesian LASSO-based analysis of the Boston Housing dataset \citep{harrison1978hedonic} and a high-dimensional example of using transport maps for generative modeling for the MNIST handwritten digits dataset \citep{lecun1998gradient}.

\subsection{Previous Work}
\label{subseca:previous_work}

A methodology for finding transport maps  based on ideas from  optimal transport
within the context of Bayesian inference was first proposed in
\citep{ElMoselhy2012} and expanded upon in conjunction with more traditional
MCMC-based sampling schemes in \citep{Marzouk2016,Parno2014, Parno2016,
Spantini2016}.

Our previous work used ideas from optimal transport theory to generalize the
posterior matching scheme, a mutual-information maximizing  scheme for feedback
signaling of a message point in arbitrary dimension
\citep{ma2014isit,ma2011generalizing, tantiongloc2017}. Building upon this, we considered a
relative entropy minimization formulation, as compared to what was developed in
\citep{ElMoselhy2012}, and showed that for the class of log-concave
distributions, this is a convex problem \citep{kim2013efficient}. We also
previously described a distributed framework \citep{mesa2015scalable} that we
expand upon here.

In the more traditional optimal transportation literature convex optimization
has been used to varying success in specialized cases \citep{Papadakis2013}, as
well as gradient-based optimization methods \citep{Rezende2015, Benamou2015a,
Benamou2015b}. The use of \textit{stochastic} optimization techniques in optimal
transport is also of current interest \citep{genevay2016stochastic}. In
contrast, our work below presents a specific distributed framework where
extensions to stochastic updating have been previously developed in a general
case. Incorporating them into this framework remains to be explored.

Additionally, there is much recent interest in the efficient and robust calculation of
Wasserstein \textit{barycenters} (center of mass) across partial empirical
distributions calculated over batches of samples
\citep{cuturi2014fast,claici2018stochastic}.  Wasserstein barycenters have also been applied to
Bayesian inference \citep{srivastava2015scalable}. While related, our work
focuses instead on calculating the \textit{full} empirical distribution through
various efficient parameterizations discussed below.

Building on much of this, there is growing interest in specific applications of
these transport problems in various areas
\citep{arjovsky2017wasserstein,tolstikhin2017wasserstein}. These
\textit{derived} transport problems are proving to be a fruitful alternative
approach and are the subject of intense research. The framework presented below
is general purpose and could benefit many of the derived transport problems.

Excellent introductory and references to the field can be found in 
\citep{villani2008optimal,santambrogio2015optimal}.

The rest of this paper is organized as follows: in \cref{seca:preliminaries}, we
provide some necessary definitions and background information; in
\cref{seca:distributed_push_forward}, we describe the distributed general
push-forward framework and provide several details on its construction and use;
in \cref{seca:sequential-composition}, we formulate a specialized version of the
objective specifically tailored for sequential composition; in
\cref{seca:applications}, we discuss  applications and examples of our
framework; and we provide concluding remarks in \cref{seca:discussion}.

\section{Preliminaries}
\label{seca:preliminaries}

In this section we make some preliminary definitions and provide background
information for the rest of this paper. %

\subsection{Definitions and Assumptions} 
\label{subseca:definitions}
Assume the space for sampling is given by $\aW \subset \reals^D$, a convex subset of $D$-dimensional Euclidean space.
Define the space of all probability measures on $\aW$ (endowed with the Borel sigma-algebra)
as $\probSimplex{\aW}$. If $P \in \probSimplex{W}$ admits a \emph{density} with
respect to the Lebesgue measure, we denote it as $p$.

\begin{assumption}
  \label{assmp:lebesgue}
  We assume that $P,Q \in \probSimplex{\aW}$ admit densities $p,q$ with respect
  to the Lebesgue measure.
\end{assumption}

This work is fundamentally concerned with trying to find an appropriate
\emph{push-forward} between two probability measures, $P$ and
$Q$:
\begin{definition}[Push-forward]
  Given $P, Q \in \probSimplex{W}$ we say that a map $S: \aW \to \aW$ pushes
  forward $P$ to $Q$ (denoted as $S_\# P = Q$) if a random variable $X$ with
  distribution $P$ results in $Y \triangleq S(X)$ having distribution $Q$.
\end{definition}
Of interest to us is the class of invertible and ``smooth'' push-forwards:
\begin{definition}[Diffeomorphism]
  A mapping $S$ is a diffeomorphism on $\aW$ if it is invertible, and both $S$
  and $S^{-1}$ are differentiable. Let $\D$ be the space of all diffeomorphisms
  on $\aW$.
\end{definition}
A subclass of these, are those that are ``orientation preserving'':
\begin{definition}[Monotonic Diffeomorphism]
  A mapping $S \in \D$ is orientation preserving, or monotonic, if its Jacobian
  is positive-definite:
  \[
    J_S(u) \psd 0, \quad \forall u \in \aW
  \]  
  Let $\MD \subset \D$ be the set of all monotonic diffeomorphisms on $\aW$.
\end{definition}
The Jacobian $J_S(u)$ can be thought of as
how the map ``warps'' space to facilitate the desired mapping.  Any monotonic diffeomorphism necessarily satisfies the following Jacobian equation:
\begin{lemma}[Monotonic Jacobian Equation]
  \label{lemma:mon_jacobian_push_forward}
  Let $P, Q \in \probSimplex{W}$ and assume they have densities $p$ and $q$.  Any map $S \in \MD$ for which $S \# P=Q$ satisfies the following Jacobian equation: 
  \begin{align}
    \label{eqn:defn:JacobianEqn_MD}
    p(u) = q(S(u)) \det(J_S(u)) \quad \forall u \in \aW
  \end{align}
\end{lemma}

We will now concern ourselves with two different notions of ``distance'' between probability measures.
\begin{definition}[KL Divergence]
Let $P, Q \in \probSimplex{W}$ and assume they have densities $p$ and $q$. The Kullback-Leibler (KL) divergence, or relative entropy, between $P$ and $Q$ is given by
	\begin{eqnarray}
 D(P\|Q)&=& \bE_{P}\brackets{\log\frac{p(X)}{q(X)}} \nonumber%
\end{eqnarray}
\end{definition}
The KL divergence is non-negative and is zero if and only if $p(u)=q(u)$ for all $u$.

\newcommand{\tp}{\tilde{p}}

\begin{definition}[Wasserstein Distance] 
For $P, Q \in \probSimplex{W}$ with densities $p$ and $q$, the  Wasserstein distance of order two between $P$ and $Q$ can be described as
\begin{align} 
    d(P,Q)^2 &\triangleq \inf\braces{ \bE_{P_{X,Y}}[\|X-Y\|^2]: X \sim P, Y \sim Q} 
\end{align}
\end{definition}

The following theorem will be useful throughout:
\begin{theorem}[\citep{brenier1987decomposition,villani2003topics}]
\label{theorem:Wasserstein:monotonicDiffeomorphism}
Under Assumption~\ref{assmp:lebesgue}, $d(P,Q)$ can be equivalently expressed as
\begin{align} 
    d(P,Q)^2 &\triangleq \inf\braces{ \bE_{P}[\|X-S(X)\|^2]: S_\# P = Q} \label{eqn:monge_problem}
\end{align}
and there is a unique minimizer $S^*$ which satisfies $S^* \in \MD$.
\end{theorem}

Note that this implies the following corollary:
\begin{corollary}
  \label{corrollary:md_exists}
  For any $P,Q$ satisfying \cref{assmp:lebesgue}, there exists a $S \in \MD$ for which $S_{\#}P = Q$, or equivalently, for which \eqref{eqn:defn:JacobianEqn_MD} holds.
\end{corollary}
\newcommand{\tpsq}{\tilde{p}}

\newcommand{\tpsu}{\tilde{p}(u;S)}
\newcommand{\tPscdot}{\tilde{P}(\cdot;S)}
\newcommand{\tqsu}{\tilde{q}(u;S)}
\newcommand{\tQscdot}{\tilde{Q}(\cdot;S)}

\section{KL Divergence-based Push-Forward}
\label{seca:distributed_push_forward}

In this section, we present the distributed push-forward framework that relies on our previously published relative entropy-based formulation of the measure transport problem, and discuss several issues related to its construction.

\subsection{General Push-Forward} 
\label{subseca:general_push_forward}

\begin{figure}[htbp]
  \centering
  \includegraphics[scale=0.35]{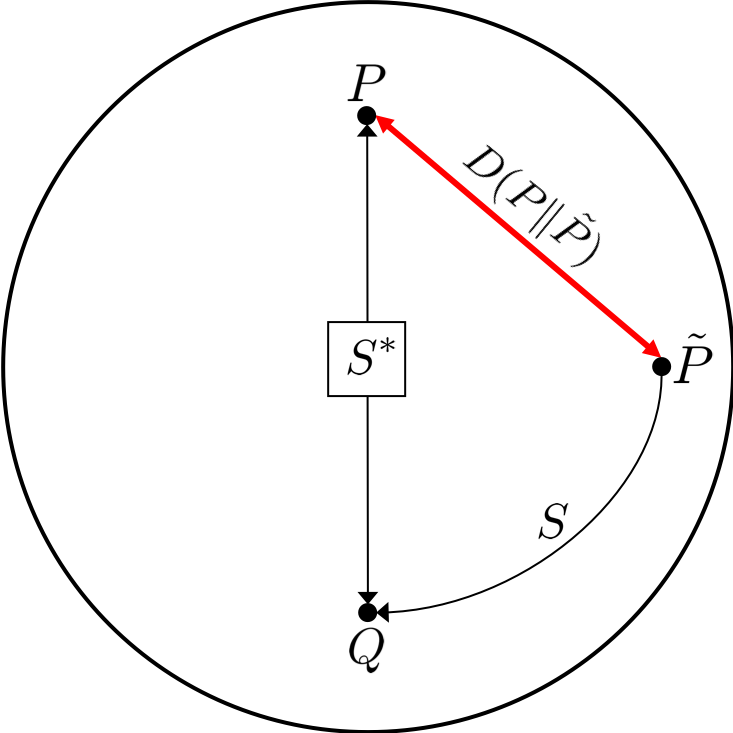}
  \caption{General Push-Forward: Probability measures $P,\tPscdot$ and $Q$ are represented as points in $\aP(\aW)$. When $Q$ is assumed to be constant, an arbitrary map $S \in \MD$ can be thought of as inducing a distribution $\tPscdot$.  Thus, $S$ pushes $\tPscdot$ to $Q$ (the solid black line labeled $S$ in the figure). The problem of interest is to then find the $S$ that minimizes the distance between the true $P$, and $\tPscdot$.  The optimal map $S^*$, represented by the center line,  pushes $P$ to $Q$. }
  \label{fig:gen-push-forward}
\end{figure}

According to \cref{lemma:mon_jacobian_push_forward}, a monotonic diffeomorphism pushing $P$
to $Q$ will necessarily satisfy the Jacobian equation
\eqref{eqn:defn:JacobianEqn_MD}. Note that although we think of a map $S$ as
pushing \emph{from} $P$ \emph{to} $Q$, we have written
\eqref{eqn:defn:JacobianEqn_MD} so that $p$ appears by itself on the left-hand
side, while $S$ is being \emph{acted on} by $q$ on the right-hand side. This
notation is suggestive of the following interpretation: If we think of the
destination density $q$ as an \emph{anchor point}, then for any \emph{arbitrary}
mapping $S \in \MD$, we can describe an \emph{induced} density for $\tpsu$
according to \cref{eqn:defn:JacobianEqn_MD} as:
\begin{align}
  \label{eqn:defn:arbJacobianEqn_MD}
  \tpsu = q(S(u)) \det(J_{S}(u)) \quad \text{ for all } u \in \aW
\end{align}
With this notation, we can interpret 
$\parenth{\tpsu: S \in \MD}$ as a parametric family of densities, and for any fixed $S \in \MD$, $\tpsu$ is a density which integrates to $1$.  We note that by construction, any $S \in \MD$ necessarily pushes $\tPscdot$ to $Q$: $S_{\#}\tPscdot=Q$.  We
can then cast the transport problem as finding the mapping $S \in \MD$
that minimizes the relative entropy between $P$ and 
the induced
$\tilde{P}$. 
\begin{align}
    \label{eqn:opt_prob}
    S^* = \argmin_{S \in \MD} \;\;  D(P\|\tPscdot) 
\end{align}
This perspective is represented visually in
\cref{fig:gen-push-forward}. 

\newcommand{\tpSX}{\tilde{p}(X;S)}

If we again make another natural assumption:
\begin{assumption}
  \label{assmp:finite-KL}
  $P$ admits a density $p$ such that: 
  \[\bE\brackets{\abs{\log p(X)}} < \infty\]
\end{assumption}
We can expand \cref{eqn:opt_prob} and combine with
\eqref{eqn:defn:arbJacobianEqn_MD} to write:
\begin{align}
    S^* &= \argmin_{S \in \MD} \;\;  D(P\|\tPscdot) \nonumber \\
        &= \argmin_{S \in \MD} \;\;  \bE_{P}\brackets{\log\frac{p(X)}{\tpSX}} \nonumber \\
         &= \argmin_{S \in \MD} \;\;  -h(p) - \bE_{P}\brackets{\log \tpSX  } \label{eqn:opt_prob_max_a} \\
         &= \argmin_{S \in \MD} \;\;  -\bE_{P}\brackets{\log \tpSX} \label{eqn:opt_prob_max_b} \\
         &= \argmin_{S \in \MD} \;\;  -\bE_{P} \brackets{ \log q(S(X))+\log \det J_{S}(X) } \label{eqn:opt_prob_max_c} 
\end{align}
where in \eqref{eqn:opt_prob_max_a}, $h(p)$ is the Shannon differential entropy
of $p$, which is fixed with respect to $S$; \eqref{eqn:opt_prob_max_b} is by
\cref{assmp:finite-KL} and Jensen's inequality implying that $\abs{h(p)}<\infty$
and the non-negativity of KL divergence; \eqref{eqn:opt_prob_max_c} is by
combining with \eqref{eqn:defn:arbJacobianEqn_MD}.

We now make another assumption for which we can guarantee efficient methods to solve for \eqref{eqn:opt_prob}.
\begin{assumption}\label{assumption:log-concave}
The density $q$ is log-concave.
\end{assumption}

We can now state the main theorem of this section \citep{kim2015tractable,mesa2015scalable}:
\begin{theorem}[General Push-Forward]
    \label{thrm:general-push-forward}
    Under Assumptions~\ref{assmp:lebesgue}—-\ref{assumption:log-concave},
    \begin{align}
        \label{eqn:general_push_forward}
        \min_{S \in \MD} \;\;  D(P\|\tPscdot) \tag{\textbf{GP}}
    \end{align}
    is a convex optimization problem.
\end{theorem}
\begin{proof}
    For any
    $S,\tilde{S}\in \MD$, we have that $J_S, J_{\tilde{S}} \psd 0$. For any
    $\lambda \in [0,1]$ we have that $\tilde{S}_\lambda \triangleq \lambda S +
    (1-\lambda)\tilde{S}$ and $J_{\tilde{S}_\lambda} = \lambda J_S + (1-\lambda)J_{\tilde{S}} \psd 0$. Since $\log\det$ is strictly concave over the space of positive definite matrices \citep{boyd2004convex}, and by assumption $\log q(\cdot)$ is concave, we have that $-\bE_{P}\brackets{\log \tpSX}$ is a convex function of $S$ on $\MD$. Existence of $S^* \in \MD$ for which $\kldist{P}{\tilde{P}(\cdot;S^*)}=0$ is given by \cref{corrollary:md_exists}.
\end{proof}
An important remark on this theorem:
\begin{remark}
    \cref{thrm:general-push-forward} does not place any structural
    assumptions on $P$. It need not be log-concave, for
    example.
\end{remark}

Beginning with \cref{eqn:opt_prob_max_c} above, we see that Problem \eqref{eqn:general_push_forward} can then be solved through the use of a Monte-Carlo approximation of the expectation, and we arrive at the following sample-based version of the formulation:

\begin{align}
    S^* &= \argmin_{S \in \MD} \;\;  \frac{1}{N} \sum_{i=1}^N \brackets{ -\log q(S(X_i)) - \log\det(J_{S}(X_i)) } \label{eqn:stoch_op}
\end{align}

where $X_i \sim p(X)$.

\subsection{Consensus Formulation} 
\label{subseca:consensus_formulation}

The stochastic optimization problem in \eqref{eqn:stoch_op} takes the general
form of:
\eq{
  & \min_S
  & & \sum_{i=1}^N f_i(S)
}
From this perspective, $S$ can be thought of as a \emph{complicating variable}.
That is, this optimization problem would be entirely separable across the sum
were it not for $S$. This can be instantiated as a \emph{global consensus}
problem:
\eq{
  & \min_S
  & & \sum_{i=1}^N f_i(S_i) \\
  & \text{s.t.}
  & & S_i - S = 0
}
where the optimization is now separable across the summation, but we must
achieve global consensus over $S$. With this in mind, we can now write a global
consensus version of \eqref{eqn:stoch_op} as:
\begin{align}
  &\min_{S_i \in \MD} \; - \frac{1}{N} \sum_{i=1}^N \log q(S_i(X_i)) + \log\det (J_{S_i}(X_i)) \nonumber\\
        & \text{s.t.} \quad S_i = S, \quad i=,1 \ldots ,N \label{eqn:stoch_opt_con}
\end{align}
In this problem, we can think of each (batch of) \emph{sample} as independently
inducing some random $\tilde{P}_i$ through a function $S_i$. The method
proposed below can then be thought of as iteratively reducing the distance
between each $\tilde{P}_i$ and the true $P$ by reducing the distance between
each $S_i$. %

This problem is still over an infinite dimensional space of functions $S
\in \MD$, however.

\subsection{Transport Map Parameterization} 
\label{subseca:transport_map_parameterization}
To address the infinite dimensional space of functions mentioned above, as in
\citep{mesa2015scalable,kim2013efficient,kim2015tractable,Marzouk2016} we
parameterize the transport map over a space of multivariate polynomial basis
functions formed as the product of $D$-many univariate polynomials of varying
degree. That is, given some $\vecx = (x_1,\ldots,x_a,\ldots,x_D) \in \aW \subset \bR^D$,
we form a basis function $\phi_{\vecj}(\vecx)$ of multi-index degree $\vecj =
(j_1,\ldots,j_a,\ldots,j_D) \in \cJ$ using univariate polynomials $\psi_{j_a}$ of degree
$j_a$ as:
\begin{align*}
  \phi_{\vecj}(\vecx) = \prod_{a=1}^D\psi_{j_a}(x_a)
\end{align*}
This allows us to represent one component of $S \in \MD$ as a weighted linear
combination of basis functions with weights $w_{d,\vecj}$ as:
\begin{align*}
  S^d(\vecx) = \sum_{\vecj \in \cJ} w_{d,\vecj} \; \phi_{\vecj}(\vecx)
\end{align*}
where $\cJ$ is a set of multi-indices in the representation specifying the order
of the polynomials in the associated expansion, and $d$ denotes the $d^{th}$ component of the mapping. In order to make this problem
finite-dimensional, we must \emph{truncate} the expansion to some fixed
maximum-order $O$.
\begin{align*}
  \cJ = \braces{\vecj \in \bN^D : \sum^D_{i=1}j_i \leq O }
\end{align*}
We can now approximate any nonlinear function $S \in \MD$ as:
\begin{align*}
  S(\vecx) = W\Phi(\vecx)
\end{align*}
where $K \triangleq |\cJ|$ the size of the index-set, $\Phi(\vecx) = [\phi_{\vecj_1}(\vecx),\ldots \phi_{\vecj_K}(\vecx)]^T$, and $W \in \bR^{D\times K}$ is a matrix of weights.  

 In order to avoid confusion and in the spirit of consensus ADMM as shown in \cite{boyd2011distributed}, we introduce a consensus variable $B\triangleq W$.  With this, we can now give a finite-dimensional version of \eqref{eqn:stoch_opt_con} as:
\eqn{
  \label{eqn:stoch_op_mat}
  & \min_{W_i \in \bR^{D\times K}} \; -\frac{1}{N} \sum_{i=1}^N \brackets{ \log q(W_i\Phi(X_i)) + \log\det(W_iJ_{\Phi}(X_i)) } \\
  & \text{s.t.} \quad W_i = B, \quad W_iJ_{\Phi}(X_i) \psd 0, \quad i=1, \ldots ,N
}
with:
\eq{
  \label{eqn:sizes}
  W_i              &= \brackets{ w_1,\ldots, w_K } & D &\times K \\
  \Phi(\cdot)      &= \brackets{ \phi_{\vecj_1}(\cdot), \ldots, \phi_{\vecj_K}(\cdot) }^T & K &\times 1 \\
  \jac_\Phi(\cdot) &= \brackets{ \frac{\partial \phi_{\vecj_i}}{\partial x_j}(\cdot)}_{i,j} & K &\times D
}
where we have made explicit
the implicit constraint that $\det(J_S) \geq 0$ by ensuring that $BJ_{\Phi}
\psd 0$. We now provide two important remarks:
\begin{remark}
  \label{remark:any_poly}
  In principle, any basis of polynomials whose finite-dimensional approximations
  are sufficiently dense over $\aW$ will suffice.
  In  applications where $P$ is assumed known, the basis functions are chosen to be orthogonal with respect to
  the reference measure $P$:
  \begin{align*}
    \int_{\aW} \phi_{\vecj}(\vecx) \; \phi_{\veci}(\vecx) \; p(x)dx = \mathbbm{1}_{\veci=\vecj}
  \end{align*}
  Within the context of Bayesian inference, for instance, this greatly simplifies computing conditional expectations, corresponding conditional moments, etc. \citep{schoutens2000stochastic}.  
\end{remark}

\begin{remark}
When it is important to ensure that the approximation satisfies the properties of a diffeomorphism, we can project $S(\vecx)$ onto $\MD$ with solving a quadratic optimization problem, as discussed in \cref{subseca:appendix:cc}.
\end{remark}

We also note that the polynomial representation presented above is chosen to
best approximate a transport map, independent of a specific application or
representation of the data (Fourier, wavelet, etc.). As mentioned in 
\cref{remark:any_poly} above, in principle any dense basis will suffice. 

\subsection{Distributed Push-Forward with Consensus ADMM} 
\label{subseca:distributed_consensus_admm}

In this section we will reformulate 
\eqref{eqn:stoch_op_mat} within the framework of the alternating direction method of multipliers (ADMM), and provide our main
result, \cref{corollary:distributed_push_forward}.
\subsubsection{Distributed Algorithm} 
\label{subsubseca:distributed_algorithm}

Using ADMM, we can reformulate \eqref{eqn:stoch_op_mat} as a global consensus problem to accommodate a parallelizable implementation. For
notational clarity, we write $\Phi_i \triangleq \Phi(X_i)$ and $J_i \triangleq
J_{\Phi}(X_i)$. We then introduce the following auxiliary variables:
\begin{equation*}
 \quad B\Phi_i \triangleq p_i, \quad BJ_i \triangleq Z_i
\end{equation*}
We can now write \eqref{eqn:stoch_opt_con} as:
\eq{
  &\min_{\{W,Z,p\}_i,B} && \frac{1}{N}\sum_{i=1}^N - \log q(p_i) -\log \det Z_i  + \half \rho \|W_i - B \|_2^2  \\
                &&& + \frac{1}{N} \sum_{i=1}^N \half \rho \| B\Phi_i - p_i\|_2^2 +  \half \rho \| BJ_i - Z_i\|_2^2  \\
  &\text{s.t.}  && B\Phi_i = p_i:  \quad\quad\quad\quad \gamma_i \quad (D \times 1) \\
                &&& BJ_i = Z_i:  \quad\quad\quad\quad \lambda_i \quad (D \times D) \\
                &&& W_i - B = 0:  \quad\quad\quad\alpha_i \quad (D \times K) \\
                &&& Z_i \succeq 0
}
where in the feasible set, we have denoted the Lagrange multiplier that will be
associated with each constraint to the right. 

Although coordinate descent algorithms solve for one variable at a time while
fixing the others and can be extremely efficient, they are not always guaranteed
to find the globally optimal solution \cite{wright2015coordinate}. Using the
consensus formulation of ADMM above, we consider a problem formulation with the
same global optimum which contains \textit{quadratic penalties} associated with
equality constraints in the objective function and constraints still imposed.
The consensus formulation has the key property that its Lagrangian, termed the
"augmented Lagrangian" \cite{boyd2011distributed}, can be globally minimized
with coordinated descent algorithms for any $\rho > 0$.  Note that when $\rho =
0$, the augmented Lagrangian is equivalent to the standard (unaugmented)
Lagrangian associated with \eqref{eqn:stoch_op_mat}.

We can now raise the constraints to form the fully-penalized augmented
Lagrangian as:
\beqas
L_\rho(W,Z,p,B; \gamma, \lambda,\alpha) 
 &=&  \frac{1}{N}\sum_{i=1}^N -\log q (p_i) -\log \det Z_i  \\
 &+&   \frac{1}{N}\sum_{i=1}^N \half \rho \|W_i - B \|_2^2  + \half \rho \| B\Phi_i - p_i\|_2^2 \\
 &+&  \frac{1}{N}\sum_{i=1}^N   \half \rho \| BJ_i - Z_i\|_2^2 + \gamma_i^T(p_i-B\Phi_i )\\
 &+&  \frac{1}{N}\sum_{i=1}^N  \tr \parenth{\lambda_i^T (Z_i - BJ_i)}+ \tr \parenth{\alpha_i^T (W_i - B)}
 \label{eqn:penalized_lag}
\eeqas

The key property we leverage from the ADMM framework is the ability to minimize
this Lagrangian across each optimization variable \itt{sequentially}, using only the
\itt{most recently} updated estimates. After simplification (details can be found
in the Appendix), the final ADMM update equations for the remaining variables
are:
\begin{subequations} 
  \label{eqn:ADMMfin}
  \begin{align}
    B^{k+1}         &= \cB_i \cdot \cB_{s} \label{eqn:ADMMfin:B} \\  
    W_i^{k+1}       &= -\frac{1}{\rho} \alpha^{k}_i+B^{k+1} \label{eqn:ADMMfin:F}  \\
    Z_i^{k+1}       &= Q \tilde{Z}_{i} Q^T \label{eqn:ADMMfin:Z} \\
    \gamma_i^{k+1}  &= \gamma_i^k + \rho (p_i^{k+1}-B^{k+1}\Phi_i) \label{eqn:ADMMfin:gamma} \\
    \lambda_i^{k+1} &= \lambda_i^k + \rho(Z_i^{k+1} - B^{k+1}J_i) \label{eqn:ADMMfin:lambda} \\
    \alpha_i^{k+1}  &= \alpha_i^k + \rho(W_i^{k+1}-B^{k+1}) \label{eqn:ADMMfin:alpha} \\
    p_i^{k+1}       &= \argmin_{p_i} - \log q(p_i) + \textrm{pen}(p_i) \label{eqn:ADMMfin:p}
  \end{align}
\end{subequations}
We look first at the consensus variable $B^{k+1}$. We can separate its update
into two pieces: a static component $\cB_s$, and an iterative component $\cB_i$:
\begin{subequations}
  \label{eqn:ADMM:B_details}
  \begin{align}
    \cB_i   &= \frac{1}{N}\sum_{i=1}^N \brackets{ \rho \paren{ W_i^k + p^{k}_i \Phi_i^{T}+Z_i^{k}J_i^{T} } + \gamma_i^{k}\Phi_i^{T}+\lambda_{i}^{k}J_i^{T}+\alpha_{i}^{k} } \\
    \cB_s   &= \brackets{\rho \paren{ I + \frac{1}{N}\sum_{i=1}^{N} \Phi_i \Phi_i^{T} + J_i J_i^{T} } }^{-1}
  \end{align}
\end{subequations}
The consensus variable can then be thought of as averaging the effect of all
other auxiliary variables, and forming the current best estimate for consensus
among the distributed computational nodes. 

The $p$-update is the only remaining minimization step that cannot necessarily be solved in
closed form, as it completely contains the structure of the $q$ density. In its
penalization, all other optimization variables are fixed:
\begin{align*}
  \mathrm{pen}(p_i) = \half \rho \| B^{k+1}\Phi_i - p_i \|_2^2 + \gamma_{i}^{kT}(p_i - B^{k+1}\Phi_i)
\end{align*}
The formulation of \eqref{eqn:ADMMfin} has the following desirable properties:
  \bitm
    \item \cref{eqn:ADMMfin:B,eqn:ADMMfin:F,eqn:ADMMfin:Z,eqn:ADMMfin:gamma,eqn:ADMMfin:lambda,eqn:ADMMfin:alpha} admit closed form solutions. In particular, \cref{eqn:ADMMfin:F,eqn:ADMMfin:gamma,eqn:ADMMfin:lambda,eqn:ADMMfin:alpha} are simple arithmetic updates;
    \item \cref{eqn:ADMMfin:p} is a penalized $d$-dimensional-vector convex optimization problem that entirely captures the structure of $Q$. In particular, any changes to the problem specifying a different structure of $Q$ will be entirely confined in this update; furthermore, algorithm designers can utilize any optimization procedure/library of their choosing to perform this update.
  \eitm

With this, we can now give an efficient, distributed version of the general push-forward theorem:
\begin{corollary}[Distributed Push-Forward] 
  \label{corollary:distributed_push_forward}
  Under \cref{assmp:lebesgue} and \cref{assumption:log-concave},
  \eqn{
    \label{eqn:stoch_opt_mat_con}
    &\min_{W_i \in \bR^{d\times K}} & & -\frac{1}{N} \sum_{i=1}^N \log q(W_i\Phi_i) + \log\det (W_iJ_i) \\
    & \text{s.t.} & & W_i = W, \;\; WJ_i \psd 0 \quad i=1,\ldots,N
  }
  is a convex optimization problem. 
\end{corollary}

\begin{remark}
  ADMM convergence's  properties are robust to inaccuracies in the initial
  stages of the iterative solving process \citep{boyd2011distributed}.
  Additionally several key concentration results provide very strong bounds for
  averages of random samples from log-concave distributions, showing that the
  approximation is indeed robust \citep[Thrm 1.1, 1.2]{bobkov2011concentration}.
\end{remark}

The above framework, under natural assumptions, facilitates the efficient,
distributed and scalable calculation of an optimal map that pushes forward some
$P$ to some $Q$.

\subsection{Structure of the Transport Map} 
\label{subseca:structure_of_the_transport_map}
An important consideration in ensuring the construction of transport maps is
efficient is their underlying \emph{structure}. In
\cref{subseca:transport_map_parameterization} we described a parameterization of
the transport map through the multi-index set $\cJ$ - the indices of polynomial orders
involved in the expansion. However, this parameterization tends to be unfeasible to use in high dimension or with high order polynomials due to the exponential rate at which the number of polynomials increases with respect to these two properties. 

In \citep{Marzouk2016}, two less expressive, but more computationally feasible map structures that can be used to generate the transport map were discussed, which we briefly reproduce here, along with some useful properties.  For more specific details and examples of multi-index sets pertaining to each mode for implementation purposes, see \cref{subseca:appendix:c}

The first alternative to the map pertaining to the fully-expressive mapping is
the Knothe-Rosenblatt map \citep{Emi2013}, which our group also previously used
within the context of generating transport maps for  optimal message point
feedback communication \citep{ma2011generalizing}.  Here, each component of the
output, $S^d$, is only a function of the first $d$ components of the input,
resulting in a mapping that is \emph{lower-triangular}. Both the
Knothe-Rosenblatt and dense mapping described above perform the transport from
one density to another, but with \emph{different} geometric transformations. An
example of these differences can be found in Figures 3 and 4 of
\citep{ma2011generalizing}.

A Knothe-Rosenblatt arrangement gives the following multi-index set (note that
the index-set is now sub-scripted according to dimension of the data denoting
the dependence on data component):

\eq{
    \cJ^{KR}_d = \braces{ \vecj \in \bN^D : \sum^D_{i=1}j_i \leq O \wedge  j_i = 0, \forall i > d}, d=1,\ldots,D
}

An especially useful property of this parameterization is the following identity for the Jacobian of the map:

\begin{align}
\log\det(J_{S}(X_i)) = \sum_{d=1}^D\log\partial_dS^d(X_i)
\end{align}

where $\partial_dS^d(X_i)$ represents the partial derivative of the $d^{th}$ component of the mapping with respect to the $d^{th}$ component of the data, evaluated at $X_i$.

Furthermore, the positive-definiteness of the Jacobian can equivalently be enforced for a lower-triangular mapping by ensuring the following:

\begin{align}
\partial_dS^d > 0, \quad 1 \leq d \leq D \label{eqn:kr_psd}
\end{align}

We can then write a Knothe-Rosenblatt special-case version of \cref{eqn:stoch_opt_con} as:
\begin{align}
  &\min_{S_i \in \KRMD} \; - \frac{1}{N} \sum_{i=1}^N \log q(S_i(X_i)) +  \sum_{d=1}^D\log\partial_dS_i^d(X_i) \nonumber\\
         &\text{s.t.} \quad S_i = S, \quad i=1,\ldots,N \label{eqn:stoch_opt_con_kr}
\end{align}

Indeed, we use this to our advantage in Section \ref{seca:sequential-composition}.

Finally, in the event that the Knothe-Rosenblatt mapping also proves to have too high of model complexity, an even less expressive mapping is a Knothe-Rosenblatt mapping that ignores all multivariate polynomials that involve more than one data component of the input at a time, resulting in the following multi-index set:

\eq{
    \cJ^{KRSV}_d = \braces{ \vecj \in \bN^D : \sum^D_{i=1}j_i \leq O \wedge j_i j_l = 0, \forall i\neq l \wedge j_i = 0, \forall i > d}, \quad d=1,\ldots,D
}

Although less expressive and less precise than the total order Knothe-Rosenblatt map, these maps can often still perform at an acceptable level of accuracy with respect to many problems.

\subsection{Algorithm for Inverse Mapping with Knothe-Rosenblatt Transport} 
\label{subseca:algorithm_for_inverse_mapping}
It may be desirable to compute the inverse mapping of a given sample from $Q$, that is, $S^{-1}(X), X \sim Q$.  When the forward mapping $S$ is constrained to have Knothe-Rosenblatt structure, and a polynomial basis is used to parameterize the mapping, the process of inverting a sample from $Q$ reduces to solving a sequential series of polynomial root-finding problems \citep{Marzouk2016}.  We give a more detailed implementation-based explanation of this process alongside a discussion of implementation details for the Knothe-Rosenblatt maps in \cref{subseca:appendix:d}.

\newcommand{\tpSQ}{\tilde{p}_{S,Q}}
\newcommand{\tPSQ}{\tilde{P}_{S,Q}}
\newcommand{\trho}{\tilde{\rho}}
\renewcommand{\cU}{\mathsf{U}}
\newcommand{\E}{\mathbb{E}}

\newcommand{\rhok}{\rho_{k}}
\newcommand{\rhokprev}{\rho_{k-1}}

\section{Sequential Composition of Optimal Transportation Maps}
\label{seca:sequential-composition}

\begin{figure*}[ht]
    \centering
    \includegraphics[width=\textwidth]{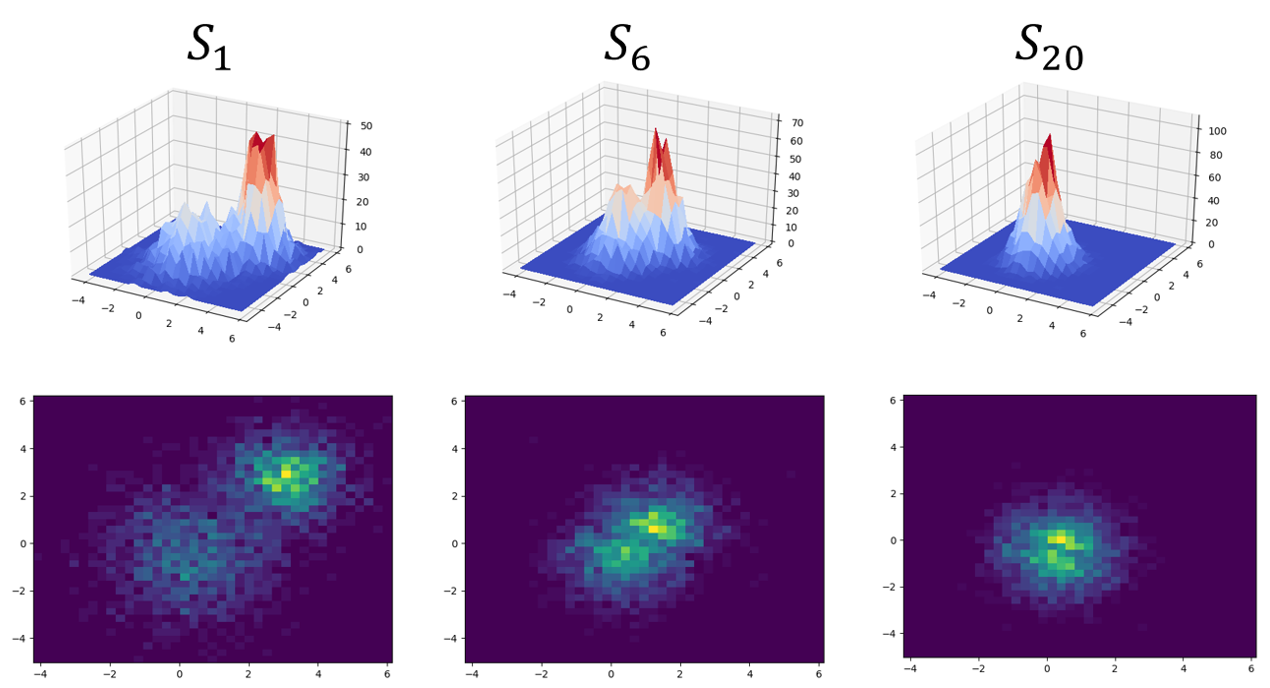}
    \caption{A visual representation of the effect a sequential composition has over the density of a set of samples shown at intermediary stages of the mapping sequence.  $P$ is a 2-dimensional bimodal distribution, and $Q$ is standard Gaussian}
    \label{fig:seq-push-forward}
\end{figure*}

In this section, we introduce a scheme for using many individually computed maps in sequential composition to achieve an overall effect of a single large mapping from $P$ to $Q$.  By using a sequence of maps to transform $P$ to $Q$ instead of a single one-shot map, one can theoretically rely on models of lower complexity to represent each map in the sequence, as although each map is, on its own, ``weak" in the sense of its ability to induce large changes in the distribution space, the combined action of many such maps together can potentially successfully transform samples as desired.  This is especially attractive for model structures that increase exponentially in complexity with problem size, such as the dense polynomial chaos structure discussed on the previous section.  This sequential composition process is visually represented in Figure \ref{fig:seq-push-forward}. 

Moving forward, we first take a brief look at a non-equilibrium thermodynamics interpretation of this methodology to further justify the use of such a scheme, and then derive a slightly different ADMM problem to implement it.  

\newcommand{\rhoinf}{\rho_\infty}
\subsection{Non-Equilibrium Thermodynamics and Sequential Evolution of Distributions}
One approach to interpreting sequential composition of maps is to borrow ideas from statistical physics, where we can interpret $q$ as the equilibrium density ($\rho_\infty)$ of particles in a system, which at time $0$ is out of equilibrium with  density $P$ (also termed $\rho_0$).   Since $q$ is an equilibrium density, it can be written as a Gibbs distribution (with temperature equal to 1 for simplicity): $q(u) \equiv \rhoinf(u) = Z^{-1} \exp \parenth{- \Psi(u)}$.
For instance, if $Q$ pertains to a standard Gaussian, then $\Psi(u)=\half u^2$.
Assuming the particles obey the Langevin equation, it is well known that the evolution of the particle density as a function of time $(\rho_t: t\geq 0)$ obeys the Fokker-Planck equation.  It was shown in \citep{Jordan1998} that the trajectory of $(\rho_t: t\geq 0)$ can be interpreted from  variational principles.  Specifically, 
\begin{theorem}[\citep{Jordan1998} Thm 5.1]\label{thm:JKO}
Define $\rho_0=p$ and $\rho_\infty=q$ and assume that $\kldist{\rho_0}{\rho_\infty} < \infty$.
For any  $h>0$, consider the following minimization problem:
\beqa
A(\rho) &\triangleq& \half d(\rhokprev,\rho)^2 + h\kldist{\rho}{\rho_\infty}\label{eqn:disc:onestep:minimization:rho:A}\\
\rhok &\triangleq& \argmin_{\rho \in \probSimplex{\aW}} A(\rho) \label{eqn:disc:onestep:minimization:rho:rhok}
\eeqa
Then as $h \downarrow 0$, the piecewise constant interpolation which equals $\rhok$ for $t \in [kh,(k+1)h)$ converges weakly in $L_1 (\reals^D)$ to $(\rho_t: t\geq 0)$, the solution to the Fokker-Planck equation.
\end{theorem}
The log-concave structure of $q$ we have exploited previously also has implications for exponential convergence to equilibrium with this statistical physics perspective:
\begin{theorem}[\citep{bakry1985diffusions}]
If $q$ is uniform log-concave, namely
\[ \nabla^2 \Psi(u) \succeq \lambda I_D\]
for some $\lambda > 0$ with $I_D$ the $D \times D$ identity matrix, then: 
\beqas
\kldist{\rho_t}{\rho_\infty} &\leq& e^{-2\lambda t} \kldist{\rho_0}{\rho_\infty}.
\eeqas
\end{theorem}
Note that if $q$ is the density of a standard Gaussian, this inequality holds with $\lambda = 1$.

\subsection{Sequential Construction of Transport Maps}
We now note that for any $h > 0$, \eqref{eqn:disc:onestep:minimization:rho:rhok} encodes a sequence $(\rhok: k \geq 0)$ of densities which evolve towards $\rhoinf \equiv q$.  For notational conciseness in this section, we will be using the subscript on $S$ to denote the position of the map in a sequence of maps.  As such, from corollary \cref{corrollary:md_exists}, there exists an $S_1 \in \MD$ for which $S_1 \# \rho_0 = \rho_1$, and more generally, for any $k \geq 0$, there exists an $S_k \in \MD$ for which $S_k \# \rho_{k-1}=\rho_k$.

\begin{lemma}
Define $B: \MD \to \reals$ as
\beqa
B(S)    &\triangleq& \half \E_{\rhokprev}\brackets{\|X-S(X)\|^2} + h \kldist{\rhokprev}{\tilde{p}(\cdot;S)} \nonumber \\
S_k &\triangleq& \argmin_{S \in \MD} B(S)  \label{eqn:defn:SkminB}
\eeqa
Then 
$A(\rho_k)=B(S_k)$ and $S_k \# \rho_{k-1}=\rho_k$.
\end{lemma}
\begin{proof}
From  the definition of $\tpSQ$ in \eqref{eqn:defn:arbJacobianEqn_MD} and the
invariance of relative entropy under an invertible transformation, any $S
\in \MD$ satisfies
\[\kldist{\rhokprev}{\tilde{p}(\cdot;S)} =  \kldist{\rhokprev}{S^{-1} \# \rho_\infty} = \kldist{S \#\rhokprev}{\rho_\infty}. \]
As such, moving forward with the proof, we will exploit how $B(S)=\tB(S)$ where
\beqas
\tB(S) \triangleq \half \E_{\rhokprev}\brackets{\|X-S(X)\|^2} + h \kldist{S \# \rhokprev}{\rho_\infty}.
\eeqas
From \cref{theorem:Wasserstein:monotonicDiffeomorphism}, $d(\rhokprev,S\#\rhokprev) \leq \E_{\rhokprev}\brackets{(X-S(X))^2}$ for any $S \in \MD$.  Also, since the relative entropy terms of $\tB(S)$ and $A(S \# \rhokprev)$ are equal, it follows that
$\tB(S) \geq A(S \# \rhokprev)$ for any $S \in \MD$.  Moreover, from \cref{corrollary:md_exists}, we have that there exists an $S_k \in \MD$  for which $S_k \# \rhokprev = \rhok$ and 
\[ \E_{\rhokprev}\brackets{\|X-S_k(X)\|^2} = d(\rhokprev,\rhok)^2.\]
Thus $\tB(S) = A(S \# \rhokprev)$.
\end{proof}

As such, a natural composition of maps underlies how a sample from $P \equiv \rho_0$ gives rise to a sample from $\rho_k$:
\beqa
\rho_k = S_k \# \rho_{k-1} = S_k \circ S_{k-1} \# \rho_{k-2} = S_k \circ \cdots \circ S_1 \# \rho_0
\eeqa

Moreover, since as $h \downarrow 0$, $\rho_k \simeq \rho_{k-1}$ and so $S_k$ approaches the identity map.  Thus for small $h >0$, each $S_k$ should be estimated with reasonable accuracy using lower-order maps.
That is, $S$ can be described as the composition of $T$ maps as 
\begin{equation}
S(x)=S_{T}\circ \ldots \circ S_2 \circ S_1(x) \label{eqn:S:compositions}
\end{equation}
for all $x \in \mathbb{R}^d$, such that each $S_i$ is of relative low-order in the polynomial chaos expansion.

Note that $B(S)$ as written above involves a sum of expectations with respect to $\rhokprev$.  Since our scheme operates sequentially, we have already estimated $S_1,S_2,\ldots,S_{k-1}$ and can generate approximate i.i.d. samples from $\rhokprev$ by first generating $(X_i: i \geq 1)$ i.i.d. from $\rho_0 \equiv p$ and constructing 
\[Z_i = S_{k-1} \circ \cdots \circ S_1(X_i),\quad i \geq 1.\]
We below will demonstrate efficient ways to solve the below convex optimization problem which replaces the expectation with respect to $\rhokprev$ instead with the empirical expectation with respect to $(Z_i: i = 1,\ldots,N)$.
\beqas
\min_{S \in \MD} \frac{1}{N} \sum_{i=1}^N \brackets{ \half \|Z_i - S(Z_i)\|^2 - h\log \tilde{p}(Z_i;S)}
\eeqas

To reiterate, we consider a distribution $\rho_{k-1}$ formed by the sequential
composition of \emph{previous} mappings as
\[\rho_{k-1} = S_{k-1}^* \circ \cdots \circ S^*_1 \# \rho_0,\]
where $\rho_0 \equiv p$. We then try to find a map
$S^*_{k}$ that pushes $\rho_{k-1}$ forward closer to $\rhoinf \equiv Q$.
Each $S_{k}$ is solved by the optimization problem \eqref{eqn:defn:SkminB}, which we term {\bf SOT}.
As the number of compositions $T$ in \eqref{eqn:S:compositions} increases, $\rho_T$ approaches $\rho_\infty$.  When $q$ is uniform log-concave, this greedy, sequential approach still guarantees exponential convergence.  

In the context of Knothe-Rosenblatt maps, for every map in the sequence we can solve the following optimization problem (in the following equation, we will be dropping the subscript $k$ that indicates the sequential map index, as the formulation is not dependent on position in the map sequence, and we will once again be replacing the subscript with $i$ to indicate the distributed variables for the consensus problem instead):
\begin{align}
   &\min_{S_i \in \KRMD} \; \theta||S_i(X_i) - X_i||^2_2 
  - \frac{1}{N} \sum_{i=1}^N \log q(S_i(X_i)) +  \sum_{d=1}^D\log\partial_dS^d(X_i) \label{eqn:stoch_opt_con_kr_seq}\\
        & \text{s.t.} \quad S_i = S, \quad \forall 0 \leq i \leq N \nonumber
\end{align}
where $\theta=h^{-1}$ can be interpreted as an inverse ``step-size'' parameter.

Though each map in the sequence must be calculated \emph{sequentially} after the previous one,
each mapping can still be calculated in the distributed framework described above. 
This implies that at each round, one could \emph{adaptively} decide the parameters for
the next-round's solve.

\subsection{ADMM Formulation for Learning Sequential Maps}
\label{subseca:admm_ot}
We now showcase an ADMM formulation for the optimal transportation-based objective function, similar in spirit to that of \cref{eqn:ADMMfin}.  

We first introduce the following conventions:

\begin{itemize}
\item $\Phi_i^d$ represents the partial derivative of $\Phi_i$ taken with respect to the $d^{th}$ component.  Therefore, $B \Phi_i^d  = \partial_d S(X_i)$, and $\partial_d S^d(X_i)$ is the $d^{th}$ component of $B \Phi_i^d$.
\item $\one_d$ represents a one-hot vector of length $D$ with the one in the $d^{th}$ position
\end{itemize}

We can then introduce a finite-dimensional representation of the transport map, as well as auxiliary variables and a consensus variable to \cref{eqn:stoch_opt_con_kr_seq} and rewrite the problem as:

\begin{align}
\begin{split}
\label{eqn:krOriginalObjectiveWithL2}
&\min_{\{W,p\}_i, \{Y,Z\}^d_i, B} \; \theta||B\Phi_i - x_i||^2_2 + \frac{1}{N}\sum_{i=1}^N -\log q(p_i) \\
&+ \frac{1}{2}\rho||W_i - B||^2_2 + \frac{1}{2}||B\Phi_i - p_i||^2_2 \\ 
&+ \sum_{d=1}^{D} -\log Z^d_i + \frac{1}{2}\rho (Y_i^d \one_d - Z^d_i)^2 + \frac{1}{2}\rho||B \Phi_i^d - Y_i^d||^2_2\\
\text{s.t} \quad & B\Phi_i = p_i \quad\quad\quad\quad \gamma_i \quad (D \times 1) \\ 
& W_i - B = 0 \quad\quad\quad\quad \alpha_i \quad (D \times K) \\
& Y_i^d \one_d = Z^d_i \quad\quad\quad\quad \beta^d_i \quad (1 \times 1) \\
& B \Phi_i^d = Y^d_i \quad\quad\quad\quad \lambda^d_i \quad (D \times 1) \\
& Z_i^d > 0
\end{split}
\end{align}

where we have once again denoted the corresponding Lagrange multipliers to the right of each constraint.  The superscript $d$ notation represents the fact that in this formulation, in addition to having separable variables for each data sample, some variables are now unique to an index over dimension as well.  For example, there are $DN$-many $Z$ variables that must be solved for.  We can now raise the constraints to form the fully-penalized Lagrangian as:

\begin{align}
\begin{split}
& L_{\rho, \theta}(W,Z,Y,p,B;\gamma, \alpha,\beta,\lambda) \\
&=  \; \theta||B\Phi_i - x_i||^2_2 + \frac{1}{N}\sum_{i=1}^N - \log q(p_i) \\
&+ \frac{1}{2}\rho||W_i - B||^2_2 + \frac{1}{2}\rho||B\Phi_i - p_i||^2_2 \\
& + \gamma_i^T(p_i - B\Phi_i) + \textbf{tr}(\alpha_i^T(F_i-B)) \\
& + \sum_{d=1}^{D} -\log Z^d_i + \frac{1}{2}\rho (Y_i^d \one_d - Z^d_i)^2 +\frac{1}{2}\rho||B \Phi_i^d - Y_i^d||^2_2 \\
& + \beta_i^d(Z_i^d - Y_i^d \one_d) + \lambda_i^{dT}(Y_i^d - B\Phi_i^d)
\end{split}
\end{align}

The final ADMM update equations for each variable are once again all closed-form, with the exception of the optimization over $p_i$. For the sake of brevity, we refer the reader to Section \ref{subseca:appendix:b} of the Appendix for the exact update equations. 

However, one notable difference between this formulation and that of Section \ref{subsubseca:distributed_algorithm} as
 noted in the previous section is that the update for $Z_i^d$ has been simplified from requiring an eigenvalue decomposition, to requiring a simple scalar computation, thus significantly reducing computation time, especially in higher dimensions.

\subsection{Scaling Parallelization with GPU Hardware} 
Given the parallelized formulations given above, we implemented our algorithm using the Nvidia CUDA API to get as much performance as possible out of our formulation, and to maximize the problem sizes we could reasonably handle, while keeping computation time as short as possible.  To test the algorithm's parallelizability, we ran our implementation on a single Nvidia GTX 1080ti GPU, as well as on a single p3.16xlarge instance available on Amazon Web Services, which itself contains 8 on-board Tesla V100 GPUs.

For this test, we have sampled synthetic data from a bimodal $P$ distribution specified as a combination of two Gaussian distributions, for a wide range of problem dimensions, specifically $D = 5,10,20,50,100,150,200$, and a constant number of samples from $P$ set to $N=1000$.  We then find a transport pushing $P$ to  $Q = \mathcal{N}(\textbf{0}, \mathbf{I})$, composed of a sequence of 10 individual Knothe-Rosenblatt maps with no mixed multivariate terms.  We then monitor the convergence of dual variables for proper termination of the algorithm.

Figure \ref{fig:gpuComparison} shows the result of this analysis.  The 1 GPU curve corresponds to performance using the single GTX 1080ti, and the AWS curve corresponds to the performance using the 8-GPU system on Amazon Web Services.  The trending of the curves shows that, as expected, as problem dimension increases, a multi-GPU system will continue to maintain reasonable computation times, at least with respect to a single-GPU system, however fewer GPU's will begin to accumulate increasingly high computational costs.  In addition, the parallelizability of our algorithm also has a subtle benefit of helping with memory-usage issues; since we can distribute samples across multiple devices, we can also subsequently distribute all corresponding ADMM variables as well.  Indeed, the single GTX 1080ti ran out of on-board memory roughly around $D=230$, whereas the 8-GPU system can go well beyond that.

\begin{figure*}[ht]
    \centering
    \includegraphics[scale=0.35]{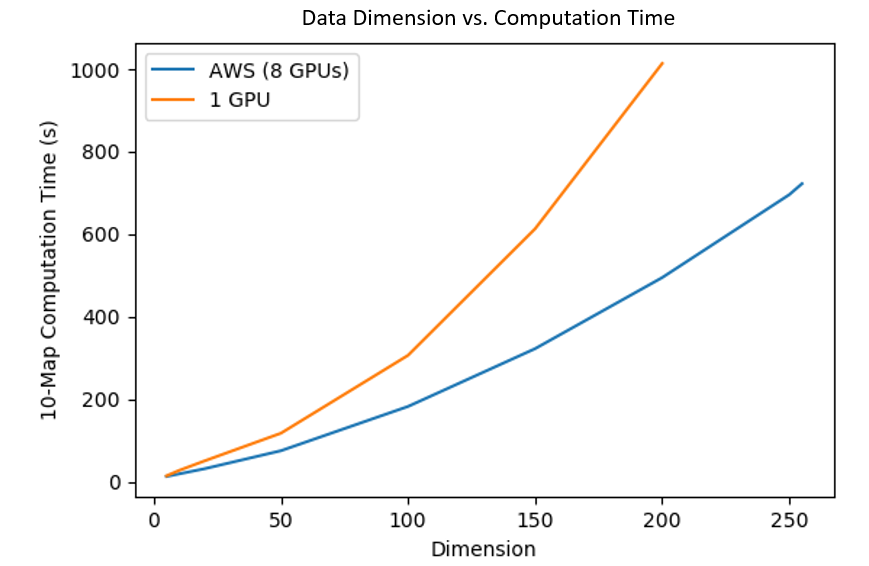}
    \caption[Example MNIST Data samples vs. Randomly Drawn Samples from $P_{digit}$ using $S^{-1}_{(digit)}$]{A comparison of using a single-GPU system vs. an 8-GPU system to compute maps in increasingly high dimension.  The trending of the two plots clearly shows the more reasonable growth in computation time of the 8-GPU system relative to the single-GPU system, as the samples from $P$ are distributed among the multiple devices}
    \label{fig:gpuComparison}
\end{figure*}

\newcommand{\fai}{FA_i}
\newcommand{\fji}{FJ_i}
\newcommand{\neghalf}{-\frac{1}{2}}
\newcommand{\mS}{\hat{S}}
\newcommand{\mSx}[1]{\hat{S}_{#1}(x)}
\section{Applications} 
\label{seca:applications}

The framework presented above is general-purpose, and works to push-forward
a distribution $P$ to a log-concave distribution $Q$. Below we discuss some interesting
applications, namely Bayesian inference and a generative model, and show results with real-world datasets.

\subsection{Bayesian Inference} 
\label{subseca:bayesian_posterior}

A very important instantiation of this framework comes when we consider $P
\equiv P_X$ to represent a prior distribution, and $Q \equiv P_{X|Y=y}$ to be a
Bayesian posterior:
\eq{
  f_{X|Y=y}(x) = \frac{f_{Y|X}(y|x) f_X(x)}{\beta_y} \label{eqn:defn:BayesRule}
}
where $\beta_y$ is a constant that does not vary with $x$, given by:
\eq{
  \beta_y = \int_{v \in \cX} f_{Y|X}(y|v) f_X(v) dv
} 
Using \cref{eqn:defn:JacobianEqn_MD} and combining with Bayes' rule above we can
write:
\eq{
  f_X(x) &= f_{X|Y=y}(S^*_{(y)}(x)) \detSysx \\
         &= \frac{f_{Y|X}(y|S^*_{(y)}(x)) f_X\parenth{S^*_{(y)}(x)}}{\beta_y}  \detSysx 
}

where the notation $S^*_{(y)}(x)$ indicates that the optimal map is found with respect to observations $y$.  We note that since $q(u) = \frac{f_X(u)f_{Y|X}(y|u)}{\beta_y}$, log-concavity of $q$ is equivalent to log-concavity of the prior density $f_X(u)$ and log-concavity of the likelihood density $f_{Y|X}(y|u)$ in $u$: the same criterion for an MAP estimation procedure to be convex.  Thus \cref{corollary:distributed_push_forward} extends to the special case of Bayesian inference; i.e. we can generate i.i.d. samples from the posterior distribution by solving a convex optimization problem in a distributed fashion.

Due to the unique way the ADMM steps were structured, this special case only
requires specifying a particular instance of \cref{eqn:ADMMfin:p}:
\eq{
  p_i^* &= \argmin_{p_i} -\log \underbrace{f_{Y|X}(y|p_i)}_{\text{likelihood}} - \log \underbrace{f_X(p_i)}_{\text{prior}} + \text{pen}(p_i)
}

\begin{remark}
This specific case establishes an important property. If the prior is chosen so
that it is easy to sample from, and the prior and likelihood are both
log-concave, then a deterministic function $S$ can be efficiently computed that
takes I.I.D samples from the prior distribution, and results in I.I.D samples
from the posterior distribution. The assumption of log-concavity is also
typically used in large-scale point estimates, though this framework goes beyond
point estimates and generates I.I.D samples form the posterior.
\end{remark}

As an instantiation of this framework, we consider a Bayesian estimation of regression parameters $x_1,...,x_d$ in the model $y= \mu \bm{1_n} +\Phi x + \epsilon$, where $y$ is the $n$-dimensional vector of responses, $\mu$ is the overall mean, $\Phi$ is a $n \times d$ regressor matrix, and $\epsilon \sim \mathcal{N}(0, \sigma^2)$ is a noise vector. The LASSO solution, 

\begin{equation} \label{eqn:l2l1}
x^*= \argmin_{x \in \mathbb{R}^d} ||y-  \Phi x ||_2^2 + \lambda||x||_1 
\end{equation}  

for some $\lambda \geq 0$ induces sparsity in the latent coefficients. The solution to \eqref{eqn:l2l1} can be seen as a posterior mode estimate when the regression parameters are distributed accordingly to a Laplacian prior.  
\begin{equation}
p(x; \lambda) = \prod_{i=1}^d \frac{\lambda}{2} e^{-\lambda |x_i|}
\end{equation}

A number of  Bayesian LASSO Gibbs samplers, which are Markov Chain Monte Carlo algorithms, are used as standard methods by which to sample from the posterior associated with problem \eqref{eqn:l2l1} \citep{park2008bayesian}, \citep{hans2009bayesian}.  

We study the accuracy and modularity of our measure transport methodology through a Bayesian LASSO analysis of the Boston Housing data set, first analyzed by Harrison and Rubinfeld \citep{harrison1978hedonic}, which is a common dataset used when comparing regression problems.  We compare our results to those obtained from utilizing a corresponding Gibbs sampler.  The Boston Housing data set consists of 13 independent predictors of the median value of owner occupied homes and 506 cases. We are interested in which combination of these 13 variables best predict the median value of homes observed in $y$, and if we can eliminate variables that do not contribute much to prediction.  The LASSO gives an automatic way for feature selection by forcing the coefficients of the predictors represented by $x^*$ to be zero. The Bayesian LASSO solution, allows for uncertainty quantification of feature selection, as we can obtain credible intervals corresponding to the coefficients of the estimates.  

We used a Gibbs sampler as presented in \citep{hans2009bayesian} where the
variance variable $\sigma^2$ is non-random.  We used 3000 samples of burn-in and
sampled 10000 samples from the posterior distribution with a fixed $\lambda$
chosen by minimizing the Bayes Information Criterion (BIC)
\citep{zou2007degrees}.  We compared that to sampling from a generated transport
map with the same $\lambda$.  We used $N=2000$ samples from a Laplace prior to
learn a fourth-order transport map of interest.  In this case, we used a
one-shot, dense map structure as described in Section
\cref{seca:distributed_push_forward}.

We note that the modularity of our problem allows for sampling from the posterior distribution of the Bayesian LASSO, by only specifying the optimization problem of \cref{eqn:ADMMfin:p} to correspond to the likelihood and prior.   

Figure \ref{fig:cred_int_Boston} shows the posterior median estimates and the corresponding 95\% credible intervals for the marginal distributions of the first 10 variables of the Boston housing data set. The LASSO estimates are shown for comparison.   Figure \ref{fig:kde_boston} shows the Kernel Density Estimates for these variables constructed with 10000 samples of either the Bayesian LASSO Gibbs sampler or the measure transported samples.  The density estimates of both methods are similar, verifying the accuracy of our methodology.  

\begin{figure}[!ht]
    \centering
    \includegraphics[scale=.5]{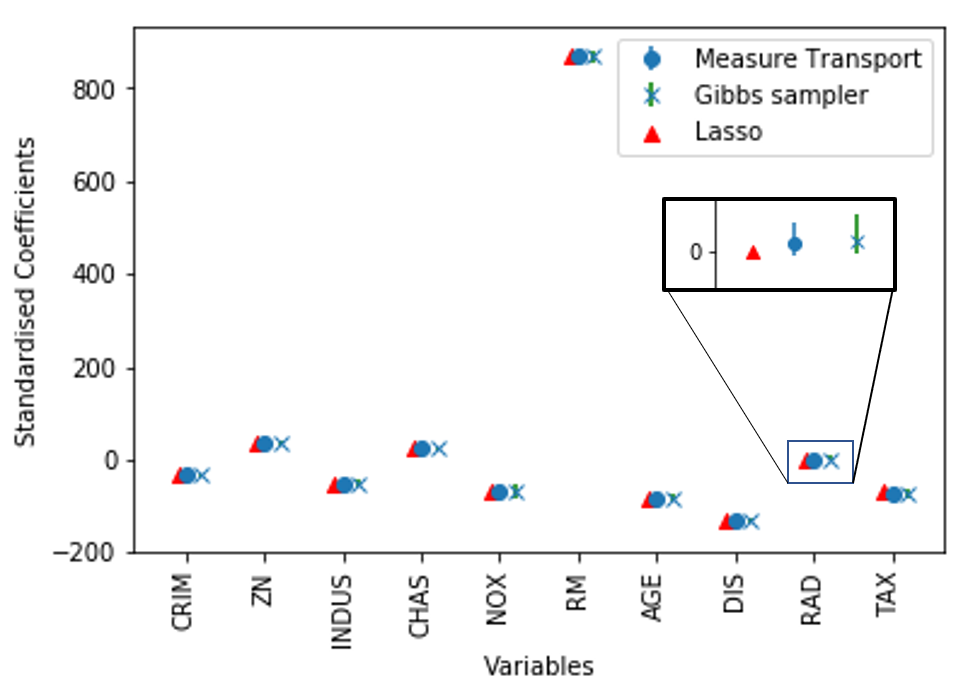}
    \caption{Posterior median Bayesian LASSO estimates and corresponding credible intervals for the ten first variables of the Boston Housing dataset. Median estimates were obtained with samples from a Gibbs sampler and a Measure Transport map. LASSO estimates are shown for comparison.}
    \label{fig:cred_int_Boston}
\end{figure}

\begin{figure}[!ht]
    \centering
    \includegraphics[width=1\textwidth]{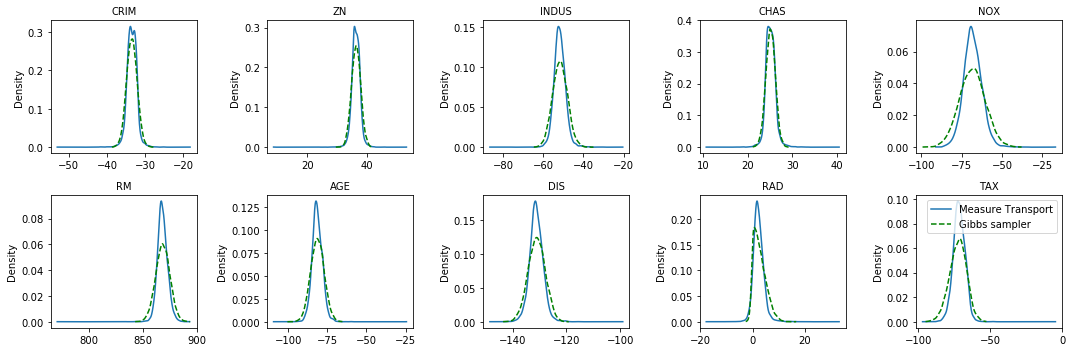}
    \caption{Kernel Density Estimate comparisons of marginal posteriors for the Boston Housing data set.}
    \label{fig:kde_boston}
\end{figure}

\subsection{High-Dimensional Maps Using the MNIST Dataset} 
\label{subsec:high_dim_mnist}
The parallelizability of our formulation of the optimal transportation-based mapping for sequential transport maps also allows us to efficiently compute maps for relatively high-dimensional data.  As a demonstration of this, we used the MNIST handwritten digits dataset \citep{lecun1998gradient} as a subject of experimentation.

Similar to the density estimation case, we assume that samples from each class of MNIST data is drawn from some $P_{digit}$, where $digit$ denotes the MNIST written digit associated with that distribution.  We then attempt to construct a (sequential) mapping, $S_{(digit)}$ that pushes $P_{digit}$ to a reference distribution, $Q = \mathcal{N}(\textbf{0}, \mathbf{I})$.  Again, similar to before, the selection of the $Q$ density to be a standard Gaussian is expressly for the purpose of analytical simplicity; $Q$ can theoretically be anything we like, so it benefits us during the generative step to select $Q$ such that it is easy to sample from. Each image in MNIST is a 28x28 pixel image, therefore after flattening each image into a vector of data, our maps operate in $\mathbf{D=784}$.  We then solve for each map $S_{(digit)}$ for every handwritten digit class in the MNIST set.

We can then treat the inverse map as a generative model; with the maps $S_{(digit)}$ in hand, we can theoretically draw samples from $Q$, and push these samples through the inverse map, $S^{-1}_{(digit)}$, resulting in randomly generated samples from $P_{digit}$.  

\cref{fig:mnist_vs_inverse} shows the result of this process using a sequential composition of 15 maps, with maximum order of the basis of each sequential map being set to 2, and each sequential map using the Knothe-Rosenblatt basis with no mixed multivariate terms from \cref{subseca:structure_of_the_transport_map}.  Our results show that even in high dimension, and even while using a relatively weak polynomial basis per sequential map, the resulting transport maps can effectively generate approximate samples from $P_{digit}$ in this way.  

\begin{figure*}[ht]
    \centering
    \includegraphics[width=12cm]{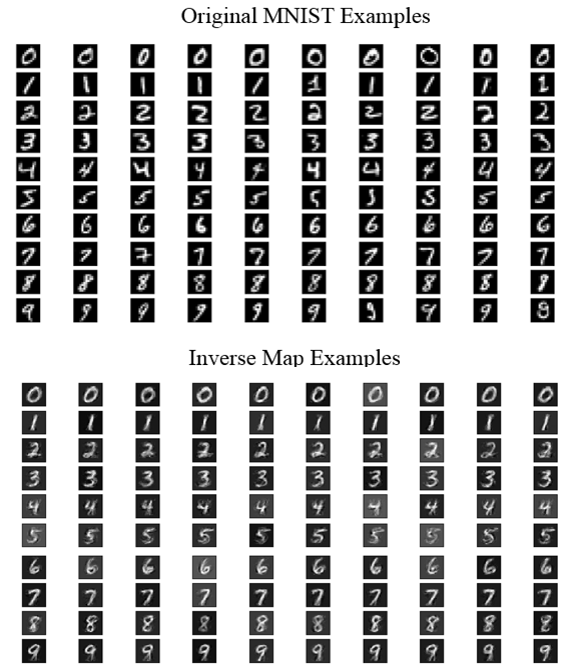}
    \caption[Example MNIST Data samples vs. Randomly Drawn Samples from $P_{digit}$ using $S^{-1}_{(digit)}$]{A comparison of original MNIST data samples vs. random samples drawn using the inverse map.  The left-most 10 columns of images pertain to randomly selected data examples from the original MNIST set, and the rightmost 10 columns of images are randomly generated by the inverse map, $S^{-1}_{(digit)}(X), X \sim Q$.  Each mapping for this example was a sequential composition of 15 maps of maximum order 2, using the Knothe-Rosenblatt mapping with no mixed terms.}
    \label{fig:mnist_vs_inverse}
\end{figure*}

\section{Discussion}
\label{seca:discussion}

In this work, we have proposed a general purpose framework for pushing
independent samples from one distribution $P$ to independent samples from
another distribution $Q$ through the \emph{efficient} and \emph{distributed}
construction of transport maps, with only independent samples from $P$, and knowledge of $Q$ up to a normalization constant. We showed that when the target distribution $Q$ is log-concave, this problem is \emph{convex}. Using ADMM, we instantiated two finite dimensional problems for finding both one-shot and sequential transport maps, and provided distributed algorithms for carrying out
the underlying optimization problems.  As our framework is distributed by nature, we can continue to take advantage of the ever-increasing availability and evolution of distributed computational resources to further speed up computation, with little to no changes to our formulation whatsoever.  

We applied our framework to a Bayesian LASSO problem, that, while it requires that the prior and likelihood to be log-concave,  is no different than existing frameworks that carry
out efficient point estimates in that regard; however, by contrast, our framework does succeed in efficiently
generating \emph{independent} samples from the actual target distribution $Q$. We
emphasize that the class of log-concave distributions is quite large and widely
used in various applications \citep{bagnoli2005log}, and that this is the same
convexity condition required for Bayesian point (MAP) estimation using many modern techniques. As such, we
have shown that from the perspective of convexity, we can go from point
estimation to fully Bayesian estimation, without requiring significantly more.

Finally, we applied our framework to a high-dimensional problem of approximating a generative model for the MNIST dataset, and provided a qualitatively striking demonstration of how well the construction of sequential transport maps can give rise to such a model. The connection and comparison of this method to other generative models, especially deep learning-based methods such as generative adversarial networks \citep{goodfellow2014generative} and variational autoencoders \citep{kingma2013auto}, remains to be explored and is the subject of future work.  We believe that this alternate form of generative model, one based on calculating a transport map that is parameterized over the space of polynomial basis functions orthogonal to the distribution of the data, stands in contrast to the black-box nature of neural networks.  Moreover, although certain works have explored the invertibility of deep neural networks \citep{lipton2017precise}, \citep{gilbert2017towards}, in general a single output of a neural network might map to multiple latent vectors.  Our transport maps, chosen over the space of diffeomorphisms, remain necessarily invertible and indeed this property is exploited in the generation of samples.  One can surmise that this invertibility leads to more tractability of the generative model.  The general connection to Optimal Transport and deep generative models is a subject of recent interest and has incited pertinent work in the literature \citep{genevay2017gan}, \citep{salimans2018improving}.

We also stress that ADMM and other related large-scale optimization methods have
many existing refinements \citep{Jordan1996, Jordan1998,Zhong2014,Azadi2014} from
which this framework would immediately benefit. Future work could explore these
refinements, and applications as approximations to non-convex problems.

Although we have established convexity of these schemes, further work needs to
be done characterizing the fundamental limits of sample complexity of this
approach, and can help guide how these architectures may possibly be soundly
implemented.  Optimizing architectures for hardware optimization, and
understanding performance-energy-complexity trade-offs, will further allow for
wider exploration of these methods within the context of emerging applications. 
\section*{Acknowledgments}
\label{sec:acknowledgments}

The authors would like to thank The Amazon Web Services Cloud Credits for Research Program for their partial and continued funding of this project with respect to cloud compute resources needed to push the current incarnation of the algorithm to the fullest.

The authors would also like to thank Sanggyun Kim, Gabe Schamberg and Alexis
Allegra for their useful discussions and comments. Additionally, the authors
thank the anonymous reviewers whose comments greatly improved the content and
presentation of this material.  
\section*{Appendix}
\label{sec:appendix}
Here we provide some additional details on several aspects of the main paper.

\subsection*{Derivation of Dense ADMM Formulation}
\label{subseca:appendix:a}

Here we show a more complete derivation of the ADMM formulation from \cref{subsubseca:distributed_algorithm}. ADMM yields the following sequential updates to the penalized Lagrangian:

\begin{subequations}
  \label{eqn:ADMM_appendix}
  \newcommand{\goleftstuff}{\!\!\!}
  \begin{align}
    \goleftstuff B^{k+1} &= \argmin_B L_\rho(W^{k},Z^{k},p^{k},B; \gamma^k, \lambda^k,\alpha^k) \label{eqn:ADMM:B} \\
    \goleftstuff W^{k+1}  &= \argmin_W L_\rho(W,Z^k,p^k,B^{k+1}; \gamma^k, \lambda^k,\alpha^k) \label{eqn:ADMM:F} \\
    \goleftstuff Z^{k+1}   &= \argmin_{Z \succ 0}  L_\rho(W^{k+1},Z,p^k,B^{k+1}; \gamma^k, \lambda^k,\alpha^k) \label{eqn:ADMM:Z} \\
    \goleftstuff p^{k+1} &= \argmin_p L(W^{k+1},Z^{k+1},p,B^{k+1}; \gamma^k, \lambda^k,\alpha^k) \label{eqn:ADMM:p} \\
    \goleftstuff \gamma_i^{k+1} &= \gamma_i^k + \rho (p_i^{k+1}-B^{k+1}\Phi_i) \quad 1 \leq i \leq n \label{eqn:ADMM:gamma} \\
    \goleftstuff \lambda_i^{k+1} &= \lambda_i^k + \rho(Z_i^{k+1} - B^{k+1}J_i) \quad 1 \leq i \leq n\label{eqn:ADMM:lambda} \\
    \goleftstuff \alpha_i^{k+1}   &= \alpha_i^k + \rho(W_i^{k+1}-B^{k+1}) \quad 1 \leq i \leq n \label{eqn:ADMM:alpha}
  \end{align}
\end{subequations}

The closed form solutions to the equations \eqref{eqn:ADMM:B},
\eqref{eqn:ADMM:F}, and \eqref{eqn:ADMM:Z} are given as follows:

Firstly, as for \eqref{eqn:ADMM:B}, the cost function $C(B^{k+1})$ is given by:
\begin{align}
C(B^{k+1}) &=  \frac{1}{N} \sum_{i=1}^N \half \rho \|W^{k}_i -B \|_F^2  + \half \rho \| B\Phi_i - p^{k}_i\|_2^2 \nonumber \\
    &+ \frac{1}{N} \sum_{i=1}^N \  \half \rho \| BJ_i - Z^{k}_i\|_F^2 + \gamma_i^{kT}(p^{k}_i-B\Phi_i ) \nonumber \\
 &+ \frac{1}{N} \sum_{i=1}^N   \tr \parenth{\lambda_i^{kT} (Z^{k}_i - BJ_i)}+ \tr \parenth{\alpha_i^{kT} (W^{k}_i -
 B)}.\label{eqn:ADMMb:costB}
\end{align}

The first-order derivative of the equation (\ref{eqn:ADMMb:costB})
in terms of $B^{k+1}$ is expressed as

\begin{align}
\frac{\partial C(B^{k+1})}{\partial B^{k+1}} &=  \frac{1}{N} \sum_{i=1}^N \rho (B-W^{k}_i)  + \rho (B\Phi_i - p^{k}_i)\Phi_i^{T} \nonumber \\
    &+ \frac{1}{N} \sum_{i=1}^N \rho ( BJ_i - Z^{k}_i )J_i^{T} - \gamma_i^{k}\Phi_i^{T} \nonumber \\
 &+ \frac{1}{N} \sum_{i=1}^N   -\lambda_i^{k}J_i - \alpha_i^{kT}.\label{eqn:ADMMb:derivB}
\end{align}

By setting the equation (\ref{eqn:ADMMb:derivB}) to zero and
expressing it in terms of $B$, we get
\begin{align}
&B\left[\rho \left ( I + \frac{1}{N}\sum_{i=1}^{N} \Phi_i \Phi_i^{T} +
J_i J_i^{T} \right )\right ]
\nonumber\\
&=\frac{1}{N}\sum_{i=1}^{N}\left[ \rho\left(W_i^k + p^{k}_i
\Phi_i^{T}+Z_i^{k}J_i^{T}\right)+\gamma_i^{k}\Phi_i^{T}+\lambda_{i}^{k}J_i^{T}+\alpha_{i}^{k}\right].
\end{align}

If we define
\eqn{
  L \triangleq \left[\rho \left ( I + \frac{1}{N}\sum_{i=1}^{N} \Phi_i \Phi_i^{T} + J_i J_i^{T} \right )\right ]
}

and 
\eqn{
M \triangleq &\frac{1}{N}\sum_{i=1}^{N}\left[ \rho\left(W_i^k + p^{k}_i
\Phi_i^{T}+Z_i^{k}J_i^{T}\right)+\gamma_i^{k}\Phi_i^{T}+\lambda_{i}^{k}J_i^{T}+\alpha_{i}^{k}\right]
}

Then we have:
\begin{align}
&B^{k+1}= M \cdot L^{-1}
\end{align}

Secondly, as for (\ref{eqn:ADMM:F}), the cost function
$C(W_i^{k+1})$ is given by

\begin{align}
C(W_i^{k+1}) =\half \rho \|W_i -B^{k+1} \|_2^2  + \tr
\parenth{\alpha^{kT}_i (W_i - B^{k+1})} \label{eqn:ADMMb:costF}
\end{align}

The first-order derivative of the equation (\ref{eqn:ADMMb:costF})
in terms of $W_i^{k+1}$ is expressed as

\begin{align}
\frac{\partial C(W_i^{k+1})}{\partial W_i^{k+1}}=\rho (W_i
-B^{k+1}) + \alpha^{k}_i .\label{eqn:ADMMb:derivF}
\end{align}
Thus,

\begin{align}
W_i^{k+1}=-\frac{1}{\rho} \alpha^{k}_i+B^{k+1}
\end{align}

Lastly, as for (\ref{eqn:ADMM:Z}), following the steps in
\citep{boyd2011distributed}, the first-order optimality condition using the
equation (\ref{eqn:ADMM:Z}) is expressed as

\begin{eqnarray}
-Z_i^{-1}+\rho(Z_i-B^{k+1}J_i)+\lambda_i^{k}=0.
\end{eqnarray}

Rewriting this, we get

\begin{eqnarray}
\rho Z_i - Z_i^{-1} = \rho B^{k+1}J_i - \lambda_i^{k}.
\label{eqn:firstopt}
\end{eqnarray}

First, take the orthogonal eigenvalue decomposition of the
right-hand side,

\begin{eqnarray}
\rho B^{k+1}J_i - \lambda_i^{k}=Q\Lambda Q^T \label{eqn:eigdec}
\end{eqnarray}
where $\Lambda=\textbf{diag}(\nu_1,...,\nu_d)$, and $Q^T Q =Q Q^T
= I$. Multiplying (\ref{eqn:firstopt}) by $Q^T$ on the left and by
$Q$ on the right gives

\begin{eqnarray}
\rho \tilde{Z}_i - \tilde{Z}_i^{-1} = \Lambda
\end{eqnarray}
where $\tilde{Z}_i=Q^T Z_i Q$. A diagonal solution of this
equation is given by

\begin{eqnarray}
\tilde{Z}_{i,(jj)} = \frac{\nu_j + \sqrt{\nu_j^2+4\rho}}{2\rho},
\end{eqnarray}
and the final solution is given as

\begin{eqnarray}
Z_i^{k+1} = Q\tilde{Z}_{i}Q^T.
\end{eqnarray}

\subsection*{Derivation of Knothe-Rosenblatt ADMM Formulation and Final Updates}
\label{subseca:appendix:b}
In similar fashion, here we outline the derivation of the ADMM formulation from \cref{subseca:admm_ot}.

First, we note that the closed-form updates for $W_i$ and $p_i$ are identical as for the original formulation.  So here we will show the derivation only for the remainder of updates. In what follows, ADMM iteration superscripts, $k$, are now enclosed in parentheses so as not to confuse them with the $d$ superscript indexing over dimension:

The cost function $C(B^{(k+1)})$ is given by:

\begin{align}
\begin{split}
C(B^{(k+1)}) &= \frac{1}{N}\sum_{i=1}^N \frac{1}{2}\rho||W_i^{(k)} - B||^2_2 + \theta||B\Phi_i - X_i||^2_2 \\
&+ \frac{1}{2}\rho||B\Phi_i - p_i^{(k)}||^2_2 + \gamma_i^{(k)T}(p_i^{(k)} - B\Phi_i) + \\
&+ \textbf{tr}(\alpha_i^{(k)T}(W_i^{(k)} - B)) \\
& + \sum_{d=1}^D \frac{1}{2}\rho||B\Phi_i^d - Y_i^{d(k)}||^2_2 + \lambda_i^{d(k)T}(Y_i^{d(k)} - B\Phi_i^d)
\end{split}
\label{eqn:ADMM_kr:B_cost}
\end{align}

Taking the first-order derivative of \cref{eqn:ADMM_kr:B_cost} and setting to 0, we arrive at the following expression:

\begin{align}
\begin{split}
& B\left[\rho(\textbf{I} + \frac{1}{N}\sum_{i=1}^N \Phi_i \Phi_i^T + \frac{2 \theta}{\rho} \Phi_i \Phi_i^T +  \sum_{d=1}^D \Phi_i^d \Phi_i^{dT})\right] \\
&= \frac{1}{N}\sum_{i=1}^N \rho W_i^{(k)} + \rho p_i^{(k)} \Phi_i^T + 2 \theta X_i \Phi_i^T + \gamma_i^{(k)} \Phi_i^T + \alpha_i^{(k)T} \\
&+ \sum_{d=1}^D \rho Y_i^{d(k)} \Phi_i^{dT} + \lambda_i^{d(k)} \Phi_i^{dT} 
\end{split}
\end{align}

If we define
\eqn{
  \cB_{s} \triangleq \rho\left(\textbf{I} + \frac{1}{N}\sum_{i=1}^N \Phi_i \Phi_i^T + \frac{2 \theta}{\rho} \Phi_i \Phi_i^T +  \sum_{d=1}^D \Phi_i^d \Phi_i^{dT}\right)
}

and

\eqn{
  \cB_i &\triangleq \frac{1}{N}\sum_{i=1}^N \rho W_i^{(k)} + \rho p_i^{(k)} \Phi_i^T + 2 \theta X_i \Phi_i^T + \gamma_i^{(k)} \Phi_i^T + \alpha_i^{(k)T} \\
&+ \sum_{d=1}^D \rho Y_i^{d(k)} \Phi_i^{dT} + \lambda_i^{d(k)} \Phi_i^{dT} 
}

then we have:
\begin{align}
\begin{split}
& B^{(k+1)} = \cB_i \cdot \cB_{s}^{-1}
\end{split}
\end{align}

The loss function associated with $Z_i^d$ for a given $i$ and $d$ is the following:

\begin{align}
\begin{split}
C(Z_i^{d(k+1)}) &= -\log Z_i^d + \frac{1}{2}\rho(Y_i^{d(k)} \one_d - Z_i^d)^2 \\
&+\beta_i^{d(k)}(Z_i^d - Y_i^{d(k)} \one_d) \nonumber
\end{split}
\end{align}

Taking the derivative and setting to 0, we get the following quadratic expression:

\begin{equation}
\rho Z_i^{d2} + (\beta_i^{d(k)} - \rho Y_i^{d(k)} \one_d)Z_i^d - 1 = 0
\end{equation}

As we would like $Z_i^{d(k+1)}$ to be greater than 0 according to our constraints, we set the closed-form solution to the positive root of this quadratic equation:

\begin{equation}
Z_i^{d(k+1)} = \frac{\rho Y_i^{d(k)} \one_d - \beta_i^{d(k)} + \sqrt{(\rho Y_i^{d(k)} \one_d - \beta_i^{d(k)})^2 + 4 \rho}}{2 \rho}
\end{equation}

The loss function associated with $Y_i^d$ for a given $i$ and $d$ is the following:

\begin{align}
\begin{split}
C(Y_i^{d(k+1)}) &= \frac{1}{2}\rho(Y_i^d \one_d - Z_i^{d(k+1)})^2 + \frac{1}{2}\rho ||B^{(k+1)}\Phi_i^d - Y_i^d||^2_2 \\
& + \beta_i^{d(k)}(Z_i^{d(k+1)} - Y_i^d \one_d) + \lambda_i^{d(k)T}(Y_i^d - B^{(k+1)}\Phi_i^d)
\end{split}
\end{align}

Taking the derivative with respect to $Y_i^d$ and setting to 0, we get the following expression:

\begin{align}
Y_i^{d(k+1)} &= (\rho Z_i^{d(k+1)} \one_d^T + \rho B^{(k+1)} \Phi_i^d + \beta_i^{d(k)} \one_d^T - \lambda_i^{d(k)T}) \\
& \cdot (\rho \one_d \one_d^T + \rho \textbf{I})^{-1} \nonumber
\end{align}

Finally, our complete set of updates is:

\begin{subequations} 
  \label{eqn:ADMMfin_kr_ot}
  \begin{align}
    B^{(k+1)}         &= \cB_i \cdot \cB_{s} \label{eqn:ADMMfin_kr_ot:B} \\  
    W_i^{(k+1)}       &= -\frac{1}{\rho} \alpha^{(k)}_i+B^{(k+1)} \label{eqn:ADMMfin_kr_ot:F}  \\
    Z_i^{d(k+1)}       &= \frac{\rho Y_i^{d(k)} \one_d - \beta_i^{d(k)} + \sqrt{(\rho Y_i^{d(k)} \one_d - \beta_i^{d(k)})^2 + 4 \rho}}{2 \rho} \\
    Y_i^{d(k+1)}      &= (\rho Z_i^{d(k+1)} \one_d^T + \rho B^{(k+1)} \Phi_i^d + \beta_i^{d(k)} \one_d^T - \lambda_i^{d(k)T})\\
    &\cdot (\rho \one_d \one_d^T + \rho \textbf{I})^{-1} \nonumber\\
    \gamma_i^{(k+1)}  &= \gamma_i^{(k)} + \rho (p_i^{(k+1)}-B^{(k+1)}\Phi_i) \label{eqn:ADMMfin_kr_ot:gamma} \\
    \alpha_i^{(k+1)}  &= \alpha_i^{(k)}+ \rho(W_i^{(k+1)}-B^{(k+1)}) \label{eqn:ADMMfin_kr_ot:alpha} \\
    \lambda_i^{d(k+1)} &= \lambda_i^{d(k)} + \rho(Y_i^{d(k+1)} - B^{(k+1)} \Phi_i^d) \label{eqn:ADMMfin_kr_ot:lambda} \\
    \beta_i^{d(k+1)} &= \beta_i^{d(k)} + \rho(Z_i^{d(k+1)} - Y_i^{d(k+1)} \one_d) \label{eqn:ADMMfin_kr_ot:beta} \\
    p_i^{(k+1)}       &= \argmin_{p_i} - \log q(p_i) + \textrm{pen}(p_i) \label{eqn:ADMMfin_kr_ot:p}
  \end{align}
\end{subequations}

where the $p_i$ update can once again be performed using any number of appropriate optimization techniques.

\subsection*{Transport Map Multi-Indices Details}
\label{subseca:appendix:c}
In this section, we give a few concrete examples of the various multi-index-sets presented in \cref{subseca:structure_of_the_transport_map} for clarification in practical use-cases, as well as for actual implementation purposes.  

In the case of a dense map, recall the index set:

\begin{align*}
  \cJ^{D} = \braces{\vecj \in \bN^d : \sum^d_{i=1}j_i \leq O }
\end{align*}

For example, in the case where $D=O=3$, the resulting index set will have the following form:
\[
  \cJ^{D} = 
  \left[ \begin{smallmatrix}
    0 & 1 & 2 & 3 & 0 & 0 & 0 & 1 & 1 & 2 & 0 & 0 & 0 & 1 & 1 & 2 & 0 & 0 & 1 & 0 \\
    0 & 0 & 0 & 0 & 1 & 2 & 3 & 1 & 2 & 1 & 0 & 1 & 2 & 0 & 1 & 0 & 0 & 1 & 0 & 0 \\
    0 & 0 & 0 & 0 & 0 & 0 & 0 & 0 & 0 & 0 & 1 & 1 & 1 & 1 & 1 & 1 & 2 & 2 & 2 & 3 
  \end{smallmatrix} \right]
\]

where every $\mathbf{j}^{th}$ column is one $D$-long multi-index for a single multivariate polynomial basis term, $\phi_\mathbf{j}$.  

The size of this set $K \triangleq |\cJ^D|$ for any given maximum polynomial order
$O$ is:
\begin{align*}
  K = \binom{D+O}{O}
\end{align*}

In the case of the Total Order Knothe-Rosenblatt map, the index set is:

\eq{
    &\cJ^{KR}_d = \\
    & \braces{ \vecj \in \bN^d : \sum^d_{i=1}j_i \leq O \wedge  j_i = 0, \forall i > d}, d=1,\ldots,D
}

In this case, the size of the set $K_d \triangleq |\cJ^{KR}_d|$ becomes dependent on the component of the mapping.  

Revisiting our previous example with $D=O=3$ we have:
\eq{
  \cJ_1^{KR} &= \braces{ \vecj \in \bN^3 : \sum_{i=1}^{3} j_i \leq O \wedge j_2 = j_3 = 0} \\
  &= \left[ 
  \begin{smallmatrix}
    0 & 1 & 2 & 3 \\
    0 & 0 & 0 & 0 \\
    0 & 0 & 0 & 0 
  \end{smallmatrix}\right] \\
    \cJ_2^{KR} &= \braces{ \vecj \in \bN^3 : \sum_{i=1}^{3} j_i \leq O \wedge j_3 = 0} \\
  &= \left[ 
  \begin{smallmatrix}
    0 & 1 & 2 & 3 & 0 & 0 & 0 & 1 & 1 & 2\\
    0 & 0 & 0 & 0 & 1 & 2 & 3 & 1 & 2 & 1\\
    0 & 0 & 0 & 0 & 0 & 0 & 0 & 0 & 0 & 0
  \end{smallmatrix}\right] \\
  \cJ_3^{KR} &= \braces{ \vecj \in \bN^3 : \sum_{i=1}^{3} j_i \leq O} \\
  &= \left[ \begin{smallmatrix}
    0 & 1 & 2 & 3 & 0 & 0 & 0 & 1 & 1 & 2 & 0 & 0 & 0 & 1 & 1 & 2 & 0 & 0 & 1 & 0 \\
    0 & 0 & 0 & 0 & 1 & 2 & 3 & 1 & 2 & 1 & 0 & 1 & 2 & 0 & 1 & 0 & 0 & 1 & 0 & 0 \\
    0 & 0 & 0 & 0 & 0 & 0 & 0 & 0 & 0 & 0 & 1 & 1 & 1 & 1 & 1 & 1 & 2 & 2 & 2 & 3 
  \end{smallmatrix} \right]
} 

In contrast to a dense mapping, this construction yields a weight matrix that has
\begin{align}
  |\cJ^{KR}_d| = \binom{d+O}{O}
\end{align}
many non-zero weights per row $d$, for a total of:
\begin{align}
  \sum_{d=1}^D \binom{d+O}{O}
\end{align}
non-zero weights. In terms of implementation, note that we can enforce a lower-triangular structure of the mapping simply by constructing $\Phi$ according to the full index set ordering of $\cJ^{KR}_D$, and constraining the coefficient matrix $W$ to have zeros embedded with the following structure:

\begin{definition}[Lower-Triangular Weight Matrix]
  A weight matrix $W \in \bR^{D \times K}$ corresponds to a lower-triangular
  transport map if it can be expressed as:
  \begin{align*}
    W = 
    \begin{bmatrix}
          \weightRow_1^T & 0 & 0 & 0 & 0 & 0 & 0 \\
          & \ldots & \weightRow_d^T & \ldots & 0 & 0 & 0 \\
          & & \ldots & \weightRow_D^T & \ldots &  &   \\
    \end{bmatrix} 
  \end{align*}
  where each $\weightRow_d$ is a vector in $\bR^{|\cJ^{KR}_d|}$. 
\end{definition}

When constructed as such, $W\Phi_i = S(X_i)$, where $S$ is a Knothe-Rosenblatt map.

In the case of the Single Univariate Knothe-Rosenblatt map, the index set becomes the following subset of $\cJ^{KR}$, again dependent on the component $d$:

\eq{
    &\cJ^{KRSV}_d = \\
    &\braces{ \vecj \in \bN^d : \sum^d_{i=1}j_i \leq O \wedge j_i j_l = 0, \forall i\neq l \wedge j_i = 0, \forall i > d}, \\
    &d=1,\ldots,D
}

Revisiting our previous example with $D = O = 3$, we have the following multi-index sets:

\eq{
  \cJ_1^{KRSV} &= \braces{ \vecj \in \bN^3 : \sum_{i=1}^{3} j_i \leq O \wedge j_2 = j_3 = 0 \wedge j_i j_l = 0, \forall i\neq l } \\
  &= \left[ 
  \begin{smallmatrix}
    0 & 1 & 2 & 3 \\
    0 & 0 & 0 & 0 \\
    0 & 0 & 0 & 0 
  \end{smallmatrix}\right] \\
    \cJ_2^{KRSV} &= \braces{ \vecj \in \bN^3 : \sum_{i=1}^{3} j_i \leq O \wedge j_3 = 0 \wedge j_i j_l = 0, \forall i\neq l } \\
  &= \left[ 
  \begin{smallmatrix}
    0 & 1 & 2 & 3 & 0 & 0 & 0\\
    0 & 0 & 0 & 0 & 1 & 2 & 3\\
    0 & 0 & 0 & 0 & 0 & 0 & 0
  \end{smallmatrix}\right] \\
  \cJ_3^{KRSV} &= \braces{ \vecj \in \bN^3 : \sum_{i=1}^{3} j_i \leq O \wedge j_i j_l = 0, \forall i\neq l } \\
  &= \left[ \begin{smallmatrix}
    0 & 1 & 2 & 3 & 0 & 0 & 0 & 0 & 0 & 0 \\
    0 & 0 & 0 & 0 & 1 & 2 & 3 & 0 & 0 & 0 \\
    0 & 0 & 0 & 0 & 0 & 0 & 0 & 1 & 2 & 3 
  \end{smallmatrix} \right]
} 

Here, all multivariate polynomial basis terms that are a product of mixed univariate polynomial terms are eliminated from the basis, resulting in a weight matrix that has:
\begin{align}
  |\cJ^{KRSV}_d| = dO + 1
\end{align}
many non-zero weights per row $d$, for a total of:
\begin{align}
  \sum_{d=1}^D dO + 1
\end{align}

non-zero weights. In terms of implementation, the 0-embedding strategy from the Total Order Knothe-Rosenblatt mapping still applies, as long as the complete index set is constructed as $\cJ^{KRSV}_D$.

\subsection*{Ensuring Diffeomorphism Properties of Parameterized Maps}
\label{subseca:appendix:cc}
For any $\tilde{S} \in \MD$ parameterized as in \cref{subseca:transport_map_parameterization}
\begin{equation}
\tilde{S}_K(x)= W\Phi(x)
\end{equation}

We must ensure that $WJ_{\Phi}(x)$ is positive definite for all $\vecx \in \aW$.  Here we will define an additional optimization problem to ensure this property.  We begin with the Euclidean Projection or the Proximal Operator of the indicator function of $\MD$. 

\begin{align}
S_W(x)= \argmin_{m(x)=W\Phi(\vecx) : J_{\Phi}(\vecx) \geq 0 } ||m(x) -W\Phi (\vecx )||^2 
\end{align}
As such, $S_W$ retains the properties of a diffeomorphism. 

\subsection*{Inverse Map Details}
\label{subseca:appendix:d}
Computing the inverse map also becomes straightforward given the above methodology of representing $B$ and $\Phi_i$

We begin by first showing the Knothe-Rosenblatt property of the map in the complete forward-map equation assuming we are using our polynomial basis representation for a given $X_i$:
\begingroup\makeatletter\def\f@size{5}\check@mathfonts
\begin{align}
	&\underbrace{
	\begin{bmatrix}
		b_{11} & b_{12} & \ldots & b_{1(K_1)} & \ldots & 0 & 0 & 0 \\
        b_{21} & b_{22} & \ldots & \ldots & b_{2(K_2)} & \ldots & 0 & 0 \\
         \vdots \\
         b_{D1} & b_{D2} & \ldots & \ldots & \ldots & \ldots & \ldots & b_{D(K_D)}\\
	\end{bmatrix}}_{B}
    \underbrace{\begin{bmatrix}
    \vertbar \\
    \Phi(X_i^1) \\
    \vertbar \\
    \vertbar \\
    \Phi(X_i^1, X_i^2)\\
    \vertbar \\
    \vdots\\
    \vertbar \\
    \Phi(X_i^1,\ldots,X_i^D)\\
    \vertbar
    \end{bmatrix}
    }_{\Phi_i} \nonumber \\
    &= 
    \begin{bmatrix}
    S(X_i^1)\\
    S(X_i^2)\\
    \vdots \\
    S(X_i^D)
    \end{bmatrix}
\end{align}\endgroup
where $X_i^d$ represents the $d^{th}$ component of the $i^{th}$ sample.  

Here, to fulfill our KR assumption, we assume that $\Phi_i$ is a column vector of the polynomial bases evaluated at $X_i$, ordered according to how many components of $X_i$ the bases are a function of.  I.e., if $K_d = |\cJ^{KR}_d|$, then $\Phi(X_i^1)$ are the first $K_1$ basis functions that are only a function of $X_1$, $\Phi(X_i^1, X_i^2)$ are the $K_2 - K_1$ basis functions that are only a function of $X_1$ and $X_2$, and so on. As such, as only the first $K_d$ elements of every $d^{th}$ row of $B$ are (potentially) non-zero, the map should have the appropriate Knothe-Rosenblatt structure by construction.

In the case where we want to invert a sample $S(X_i)$, this defines a system of equations that can be solved row by row for each component of the solution, $S(X_i^d)$, in the form of a polynomial root-finding problem for each row.  For example, we first solve for $X_i^1$, the solution of which we can call $X_i^{1*}$ by finding the (single variable) root of:

\begin{equation}
\begin{bmatrix}
		b_{11} & b_{12} & \ldots & b_{1(K_1)}
	\end{bmatrix}
    \begin{bmatrix}
    \vertbar \\
    \Phi(X_i^1) \\
    \vertbar \\
    \end{bmatrix}
    =
    S(X_i^1)
\end{equation}

Subsequently, we can solve for $X_i^2$ plugging $X_i^{1*}$ into the second equation:

\begin{equation}
\begin{bmatrix}
		b_{21} & b_{22} & \ldots & \ldots & b_{2(K_2)}
	\end{bmatrix}
    \begin{bmatrix}
    \vertbar \\
    \Phi(X_i^{1*}) \\
    \vertbar \\
    \vertbar \\
    \Phi(X_i^{1*}, X_i^2)\\
    \vertbar \\
    \end{bmatrix}
    =
    S(X_i^2)
\end{equation}

and so on.  Note that this results in $D$-many single variable root-finding problems per sample to invert, and the order of the polynomial that must be solved for will be equal to the order of the polynomial chosen to represent the basis. 
 
\bibliographystyle{apacite}
\bibliography{admm_journal_references}

\end{document}